\documentclass[11pt,a4paper]{article}
\usepackage[margin=25mm]{geometry}
\usepackage[T1]{fontenc}
\usepackage{lmodern}
\usepackage[authoryear,round]{natbib}
\usepackage{amsmath,amssymb,amsthm,booktabs,microtype}
\usepackage{graphicx,xcolor,adjustbox}
\usepackage[section]{placeins}
\usepackage{xurl}
\usepackage[hidelinks]{hyperref}
\hypersetup{
 pdftitle={Exact finite attention responses from RoPE derivatives},
 pdfauthor={Julie Huang; Maggie Chlon; Gregory Gutin; Leon Chlon},
 pdfsubject={Exact finite positional and joint key-value responses in softmax attention}}
\newtheorem{theorem}{Theorem}
\newtheorem{proposition}[theorem]{Proposition}
\newtheorem{corollary}[theorem]{Corollary}
\newcommand{\R}{\mathbb{R}}
\newcommand{\T}{^{\!\top}}
\newcommand{\Iset}{\mathcal{I}}
\newcommand{\DM}{\Delta M}
\newcommand{\norm}[1]{\left\lVert #1\right\rVert}
\newcommand{\dd}{\mathrm{d}}
\newcommand{\ph}{\varphi_1}
\title{Exact finite attention responses\\from RoPE derivatives}
\author{%
Julie Huang\textsuperscript{1}\quad
Maggie Chlon\textsuperscript{1}\quad
Gregory Gutin\textsuperscript{2}\quad
Leon Chlon\textsuperscript{1,3}\\[0.5em]
{\normalfont\small\textsuperscript{1}Hassana Labs}\\
{\normalfont\small\textsuperscript{2}Department of Computer Science, Royal Holloway, University of London}\\
{\normalfont\small\textsuperscript{3}University of Oxford}}
\date{}
\begin{document}
\maketitle

\begin{abstract}
We derive exact local responses for attention interventions, allowing
candidate edits to be scored from a cached baseline and one backward pass.
The starting point is the RoPE derivative $\partial_p z(p)=Az(p)$: its
integral gives the finite positional displacement, which we carry through
softmax without linearising either rotation or normalisation. The resulting
predictions achieve 95.36--96.52\% sign accuracy across 92,160 executed
positional edits on 768 held-out prompt sets, reducing answer-margin MAE
by 73.6--82.5\% against the positional Jacobian and 36.2--50.9\% against zero.
For simultaneous key and value edits, the same divided-difference calculus
isolates $C_{KV}=\sum_j(p'_j-p_j)\varepsilon_j$, the interaction omitted by
adding separate attributions. Retaining it reduces downstream margin MAE
by more than a factor of nine in every setting of a 5,120-intervention
sweep across two Qwen sizes, two tasks, and multiple layers; reductions
against a quadratic interaction correction are 75.9--98.5\%.
Exactness concerns the edited attention write; downstream predictions
contract that response with a baseline gradient and are evaluated by
native execution. The calculus also yields a KL certificate for local
approximation error, an exact query-conditioned gradient-step
representation whose curvature identifies attention-preserving query
directions, and minimum-norm query control. Sparse evaluation supports
candidate ranking and cache decisions with explicit local distortion criteria.
\end{abstract}

\section{Introduction}
\label{main:introduction}
Attention interventions ask how a model's answer would change if selected
tokens were read differently. Executing every candidate edit can be
expensive; attribution patching instead scores candidates with a shared
baseline gradient \citep{syed2024,kramar2024}. The quality of this ranking
depends on the local response supplied to that gradient. For finite
positional and joint cache edits, we can compute that response exactly.

We provide a reusable calculus for \emph{constructing, scoring, and
checking finite attention interventions}. Its finite response quantifies
the interaction missed by separate key/value attribution and the local
distortion caused by cache edits. Its KL certificate bounds approximation
error and can certify local score orderings. After fixing the common-logit
gauge, the positive curvature of its gradient-step representation
identifies attention-preserving query directions and supports feasible
minimum-norm control. These are checkable operations for attribution,
attention steering, and cache management. Their elementary form makes
them directly reusable in other methods (Figure~\ref{fig:workflow}).

RoPE gives this problem an explicit geometry. Holding content fixed,
position evolves under the generator $\partial_p z(p)=Az(p)$. A finite
shift rotates the key and then changes the normalised attention weights;
a positional Jacobian approximates both operations. Its accuracy depends
on angular displacement in each rotary plane. A simultaneous value edit
adds another effect: reallocating attention changes how strongly the edited
value is read. Adding separate key and value attributions omits precisely
this interaction.

We start by integrating the RoPE derivative:
\begin{equation}
 \boxed{\quad
 \partial_p z(p)=Az(p),\qquad
 z(p+\delta)-z(p)=\delta A\ph(\delta A)z(p).
 \quad}
 \label{main:central}
\end{equation}
Here $\ph(B)=\int_0^1 e^{tB}\,\dd t$ is the classical exponential divided
difference \citep{hochbruck2010}. Applying its scalar counterpart to
softmax retains the complete change in normalisation. This gives an
exact finite attention response and separates the rotary and softmax
errors made by a positional Jacobian. General key and value edits use the
same scalar calculus once their score and value increments are specified.

We focus on positional and joint-KV edits because they test complementary
mechanisms: finite rotary displacement and the coupling between attention
allocation and value content. The executed studies ask whether retaining
these local effects improves predictions beyond the edited head.
Across 92,160 executed positional edits on 768 held-out prompt sets,
predicted answer-margin changes have 95.36--96.52\% sign accuracy and
73.6--82.5\% lower MAE than the positional Jacobian. Across six
joint-cache settings, retaining $C_{KV}$ lowers downstream MAE by more
than a factor of nine against separate attribution, with every
setting-wise paired 95\% interval favouring the joint response.
The predictors share a baseline downstream gradient and are assessed
against independently executed continuations. Frequency and strength
sweeps locate both the small-edit regime where approximations agree and
the finite edits where the correction changes prediction quality.

\begin{figure}[t]
\centering
\includegraphics[width=\linewidth]{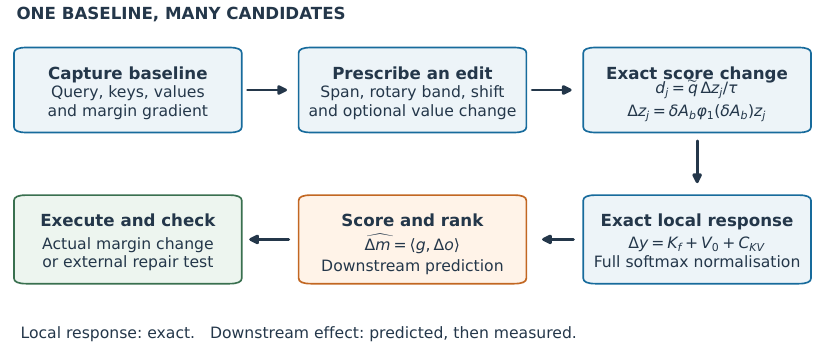}
\caption{\textbf{Construct, score, and check an attention intervention.}
The positional generator specifies a finite key displacement; exact
softmax normalisation gives its local response, including any simultaneous
value change. One baseline gradient supplies downstream scores for many
candidates. Executed continuations measure the selected edits' answer
changes; a task check evaluates a proposed repair.}
\label{fig:workflow}
\end{figure}

\section{From the RoPE derivative to a finite attention response}
\label{main:calculus}
\subsection{Integrating the positional generator}
For fixed, unscaled RoPE frequencies, write
\begin{equation}
 R(p)=e^{pA},\qquad
 A=\operatorname{diag}_{r=1}^{d/2}
 \begin{pmatrix}0&-\omega_r\\ \omega_r&0\end{pmatrix},
 \qquad z(p)=R(p)z_0 .
 \label{main:rotation}
\end{equation}
Position is extended continuously for differentiation; integer positions
are the usual evaluation points. Holding $z_0$ fixed gives
$\partial_p^n z(p)=A^nz(p)$ for every $n\ge1$. The fundamental theorem of
calculus now gives
\begin{equation}
 \begin{aligned}
 z(p+\delta)-z(p)
 &=\delta\int_0^1 Ae^{t\delta A}z(p)\,\dd t\\
 &=\delta A\ph(\delta A)z(p),\qquad
 \ph(B)=\sum_{n=0}^{\infty}\frac{B^n}{(n+1)!}.
 \end{aligned}
 \label{main:integrate}
\end{equation}
This entire function covers zero-frequency blocks without an inverse of
$A$. It also makes the approximation explicit:
$\delta Az(p)$ is the first-order term, while $\ph(\delta A)$ retains
all higher orders. On plane $r$, the relevant scale is
$\theta_r=\delta\omega_r$. The exact rotary displacement is bounded by
$2\norm{z_r}_2$, whereas the tangent has norm
$|\theta_r|\norm{z_r}_2$. A token shift can therefore be small in one band
and far outside the tangent regime in another
(Figure~\ref{fig:ropegeometry}, Appendix~\ref{sec:positionaloperators}).

\subsection{Propagating the finite shift through softmax}
Fix a deterministic attention head, its query, and a nonempty attended set.
Using row vectors, let $\widetilde q=qR(p_i)\T$,
$\widetilde k_j=k_jR(p_j)\T$, $\tau=\sqrt d$,
$s_j=\widetilde q\widetilde k_j\T/\tau$,
$p_j=e^{s_j}/\sum_k e^{s_k}$, and $y=\sum_jp_jv_j$.
Let $A_b$ retain the selected rotary blocks and be zero elsewhere.
Equation~\eqref{main:integrate} supplies the score increment for a key shifted
by $\delta$ in that band:
\begin{equation}
 \widetilde k'_j=\widetilde k_j e^{\delta A_b\T},
 \qquad
 d_j=s'_j-s_j
 =\frac{\widetilde q\,\delta A_b\ph(\delta A_b)
                         \widetilde k_j\T}{\tau}.
 \label{main:score}
\end{equation}
Unedited keys have $d_j=0$. The notation $d_j$ denotes a score increment;
$\delta$ denotes positional displacement. Define
\begin{equation}
 D=\sum_jp_je^{d_j},\qquad t_j=d_j-\log D,\qquad
 p'_j=p_je^{t_j},\qquad p'_j-p_j=p_jt_j\ph(t_j).
 \label{main:normalise}
\end{equation}
The scalar identity $e^t-1=t\ph(t)$ therefore retains the exact
softmax renormalisation.

\begin{samepage}
\begin{theorem}[Exact local response to finite edits]
\label{main:finite}
For the fixed query and attended set above, let $v'_j=v_j+\varepsilon_j$.
For any finite score increments $d_j$,
\begin{equation}
 \boxed{\quad
 \Delta y=y'-y
 =\sum_j p_jt_j\ph(t_j)(v_j-y)
       +\sum_jp'_j\varepsilon_j .
 \quad}
 \label{main:response}
\end{equation}
For a positional intervention, inserting \eqref{main:score} gives the exact
response of the full RoPE--softmax head.
\end{theorem}
\begin{proof}
Subtract the two readouts and write
$\Delta y=\sum_j(p'_j-p_j)v_j+\sum_jp'_j\varepsilon_j$.
Since the probability changes sum to zero, replace $v_j$ by $v_j-y$ in
the first sum and use \eqref{main:normalise}.
\end{proof}
\end{samepage}
This completes the chain from the derivative to a finite attention write.
The fixed-content condition matters: rotating stored keys at one readout
does not recompute the contextual features that a reordered prompt would
produce upstream. The latter can be separated by a four-cell
phase/context decomposition (Appendix~\ref{sec:positionalapplications}).

\subsection{The tangent limit and coupling between bands}
For key-only edits, the positional Jacobian substitutes
$d_j^{\rm lin}=\widetilde q\,\delta A_b\widetilde k_j\T/\tau$ and predicts
$\sum_jp_jd_j^{\rm lin}(v_j-y)$. Its local discrepancy has two explicit parts:
\begin{equation}
 \Delta y-\Delta y_{\rm Jac}
 =\underbrace{\sum_jp_jt_j\bigl(\ph(t_j)-1\bigr)(v_j-y)}_{\text{softmax remainder}}
 +\underbrace{\sum_jp_j(d_j-d_j^{\rm lin})(v_j-y)}_{\text{rotary remainder}} .
 \label{main:tworemainders}
\end{equation}
The rotary and softmax corrections are separately identifiable; both
vanish in the appropriate small-displacement limit.

Disjoint rotary bands have additive score increments,
$d^{b\cup c}=d^b+d^c$, because their generators have disjoint block support.
Their attention responses need not add. With
$F(s)=\sum_jp_j(s)v_j$, their interaction is
$F(s+d^b+d^c)-F(s+d^b)-F(s+d^c)+F(s)$.
It is the double integral of $D^2F$ over the two score directions, with
\[
 D^2F(s)[a,b]
 =\mathbb E_p[(a-\mathbb E_pa)(b-\mathbb E_pb)(v-F(s))].
\]
The coupling is therefore attributable to softmax normalisation even when
the rotary score components are independent
(Appendix~\ref{sec:finitechanges}).

\section{Joint responses, error certificates, and effective weights}
\label{main:structure}
\subsection{The interaction omitted by separate attribution}
The positional calculation specifies one important class of $d_j$.
Theorem~\ref{main:finite} also permits arbitrary prescribed key and value
edits. Decomposing its second sum gives
\begin{equation}
 \begin{gathered}
 K_f=\sum_j(p'_j-p_j)(v_j-y),\qquad V_0=\sum_jp_j\varepsilon_j,\\
 C_{KV}=\sum_j(p'_j-p_j)\varepsilon_j,\qquad
 \Delta y=K_f+V_0+C_{KV}.
 \end{gathered}
 \label{main:joint}
\end{equation}
Adding separate finite-key and baseline-weighted value responses omits
exactly $C_{KV}$. It measures how a value edit's effect changes when
attention is reallocated. Its four-cell form is
$y(p',v+\varepsilon)-y(p',v)-y(p,v+\varepsilon)+y(p,v)$.
Along a joint edit path of strength $\alpha$, its leading term is
\[
 C_{KV}(\alpha)=\alpha^2
 \sum_jp_j(d_j-\mathbb E_pd)\varepsilon_j+O(\alpha^3).
\]
The quadratic comparator in our experiments keeps this leading interaction
and the exact key response; the finite formula retains the whole interaction.

\subsection{A certificate for the local approximation}
Let $L_s=\sum_jp_jd_j(v_j-y)$ and
$w_j=p_j(e^{t_j}-1-t_j)$. Exponential convexity gives $w_j\ge0$, and
normalisation yields
\begin{equation}
 \sum_jw_j=-\sum_jp_jt_j=\operatorname{KL}(p\Vert p'),
 \qquad K_f-L_s=\sum_jw_j(v_j-y).
 \label{main:kl}
\end{equation}
For a fixed gradient $g$ after output projection $W_o$, put
$a_j=\langle g,(v_j-y)W_o\rangle$ and
$b_j=\langle g,\varepsilon_jW_o\rangle$. Then
\begin{equation}
 \begin{split}
 |\langle g,[\Delta y-(L_s+V_0)]W_o\rangle|
 \le{}&\operatorname{KL}(p\Vert p')\max_j|a_j|\\
 &+\operatorname{TV}(p,p')(\max_jb_j-\min_jb_j).
 \end{split}
 \label{main:certificate}
\end{equation}
The first term bounds the softmax remainder; the second bounds the joint
interaction. A positional-Jacobian certificate additionally accounts for
the rotary remainder in \eqref{main:tworemainders}. These are bounds on
local gradient-contracted errors, not on the network's remaining
nonlinearity. When two first-order scores differ by more than the sum of
their bounds, their exact local ordering is certified. The proof and
sparse formulas appear in Appendix~\ref{sec:finitechanges}.

\subsection{An exact gradient step in input coordinates}
Return to affine projections of an input $x_i$, and augment it as
$u_i=(x_i\T,1)\T$. RoPE gives bilinear scores
$s_{ij}=u_i\T a_{ij}$, where
$a_{ij}=\widehat W_qR(p_i)\T\widetilde k_j\T/\tau$.
Let $\mu_i$ be the uniform mean of the attended values. Applying
$e^s-1=s\ph(s)$ to the centred readout gives
\begin{equation}
 \begin{aligned}
 c_{ij}&=\ph(s_{ij})/Z_i>0,\qquad
 \DM_i=\sum_jc_{ij}a_{ij}(v_j-\mu_i),\\
 y_i&=\mu_i+u_i\T\DM_i,\qquad
 \DM_i=-\nabla_B\mathcal E_i(0),\\
 \mathcal E_i(B)&=\tfrac12\sum_jc_{ij}
       \norm{a_{ij}\T B-(v_j-\mu_i)}_2^2.
 \end{aligned}
 \label{main:effective}
\end{equation}
For every fixed head, bank, and query, this quadratic supplies the effective
matrix by an exact unit gradient step from zero. Its features, centred
values, and positive coefficients are computed directly from those inputs
and held fixed during differentiation with respect to $B$. At a fixed bank
and position, the cache supplies outer-product directions and the query
selects their coefficients; $\operatorname{rank}(\DM_i)\le d$.

The positive curvature has an operational consequence. Subtract one
reference key feature from every score feature, which leaves softmax
unchanged. Let $\Gamma_i$ be the quadratic's feature curvature restricted
to the $m$ editable input coordinates, and let $H_{\rm all}$ stack the
corresponding relative-score features. Positivity gives
\begin{equation}
 \ker\Gamma_i=\ker H_{\rm all}
 =\{h:p(x_i+h)=p(x_i)\}.
 \label{main:curvature}
\end{equation}
This holds for finite query edits with the bank and position fixed;
the kernel is independent of the query and reference key. It connects
the quadratic directly to the feasible-control problem below
(Appendix~\ref{app:curvature}). The effective matrix itself remains a
query-dependent secant: reusing an anchor $\DM_a$ omits exactly
$u_i\T(\DM_i-\DM_a)$, as characterised in
Appendix~\ref{sec:effectivemaps}.

\section{Constructing and ranking interventions}
\label{main:instrument}
The calculus supports two questions: which query edit realises a requested
attention relationship, and which candidate edit best serves a specified
objective? Both use the same captured head and exact local response.

\paragraph{Construct a requested relationship.}
For token pairs $(j_r,k_r)$, request log attention ratios $\ell_r^*$.
Apply any prescribed rotary shift first, obtaining score features $a'_j$.
Let $H$ have rows $(\bar a'_{j_r}-\bar a'_{k_r})\T$, where
$\bar a'_j$ contains the first $m$ coordinates of $a'_j$, and put
$b_r=\ell_r^*-u\T(a'_{j_r}-a'_{k_r})$. An edit $h$ of the affine
query input then satisfies the requested ratios exactly when $Hh=b$.
If this system is feasible, its unique minimum-Euclidean-norm edit is
\begin{equation}
 h_* = H^\dagger b,\qquad HH^\dagger b=b.
 \label{main:querycontrol}
\end{equation}
Thus one can set $A{:}B=4{:}1$, preserve $C{:}D$, and compensate for a
specified positional shift in one solve. Feature differences make the
construction gauge invariant; with all key pairs included, $\ker H$
is exactly the set of query directions preserving the attention distribution.
Appendix~\ref{app:querycontrol} gives the proof and infeasibility criterion.

\paragraph{Evaluate the local response.}
For score edits supported on $S$ and value edits on $T$, let
$r_j=p_jd_j\ph(d_j)$ on $S$. The exact response is
\begin{equation}
 \Delta y=
 \frac{\displaystyle
 \sum_{j\in S}r_j(v_j-y)+\sum_{j\in T}p_j\varepsilon_j
                  +\sum_{j\in S\cap T}r_j\varepsilon_j}
 {\displaystyle 1+\sum_{j\in S}r_j}.
 \label{main:sparse}
\end{equation}
The baseline readout, probabilities, and candidate-group complement
statistics are shared. Each further evaluation uses only the edited
entries. Stable evaluation uses log probabilities, \texttt{expm1}, and
positive complement masses. Preparation and reuse are timed separately.

\paragraph{Rank candidates and execute selected edits.}
Choose a scalar answer margin $m$ and a candidate family. One baseline
backward pass supplies $g_\ell=\nabla_{o_\ell}m$ at the projected write
$o_\ell=\sum_hy_{\ell h}W_{o,\ell h}$. Score each candidate $c$ by
\begin{equation}
 \widehat{\Delta m}_c
 =\left\langle g_\ell,\sum_h\Delta y_{c,h}W_{o,\ell h}\right\rangle,
 \qquad
 \Delta m_c=m_c^{\rm executed}-m^{\rm baseline}.
 \label{main:scorecandidate}
\end{equation}
Sorting the predicted improvements gives an execution order without an
additional model call per candidate. Native execution measures the actual
margin change; an external task check determines whether a proposed answer
repair succeeds. The approximation in this prediction is propagation
through the remaining network with the frozen gradient.

\paragraph{Match the decision to its objective.}
Deleting a proper token group $G$, of mass $\alpha_G<1$ and weighted
value $w_G$, gives $\Delta y_{-G}=(\alpha_Gy-w_G)/(1-\alpha_G)$.
Its projected write norm can be constrained by a specified distortion
tolerance, giving a checkable current-query eviction criterion.
For precision restoration with an original cache available, let $e$ be
the current projected-write error and $\Delta_j o$ the exact change from
restoring entry $j$. The benefit
$B_j=\norm e_2^2-\norm{e+\Delta_j o}_2^2$ ranks restorations by their
exact reduction in local squared error. Re-scoring after each restoration
defines a greedy procedure for that objective; answer repair instead
uses the margin score and executed task check.

\section{Experiments}
\label{main:experiments}
We test the predictions against executed positional and joint-cache
interventions, then evaluate ranking, precision restoration, and local
evaluation cost. Native reconstruction independently checks implementation
agreement with the identities.

\subsection{Finite positional prediction and its frequency dependence}
The positional study uses Qwen2.5-1.5B-Instruct, FP32 eager attention,
and 768 held-out prompt sets: 256 each for corrupted category mapping
(Corrupt), ordered register assignment (Ordered), and SST-2 sentiment.
A shared policy is fixed on separate calibration streams. Each set has
five variants, three layers $\{0,14,21\}$, two shifts $\{16,128\}$,
and four band choices, giving 92,160 executed edits. Keys are rotated
at one readout; query, values, and mask are fixed there. Errors are
averaged within prompt sets, and 5,000 paired prompt-level bootstrap
resamples retain all interventions in their original clusters.
Appendix~\ref{sec:positionalapplications} supplies the full protocol.

\begin{table}[t]
\centering\small
\caption{\textbf{Prediction of executed positional effects.}
Downstream margin MAE on 256 held-out prompt sets per task.
Intervals are paired 95\% prompt-bootstrap intervals for the reduction
relative to zero. All predictors use the same baseline gradient.}
\label{main:phasetable}
\begin{tabular}{lrrrr}
\toprule
Task & Finite & Jacobian & Zero & Reduction vs.\ zero\\
\midrule
Corrupt & 0.02808 & 0.16028 & 0.05724 & 50.9\% [48.2, 53.7]\\
Ordered & 0.05786 & 0.21897 & 0.09073 & 36.2\% [35.0, 37.5]\\
SST-2 & 0.03585 & 0.15426 & 0.06681 & 46.3\% [43.1, 49.9]\\
\bottomrule
\end{tabular}
\end{table}

Finite responses reduce Jacobian MAE by 73.6--82.5\%, with sign
accuracies of 95.36--96.52\% (Table~\ref{main:phasetable}).
The zero baseline measures the magnitude of the actual effect, and
finite prediction improves on it in all three tasks.

The angular scales explain where the correction matters. At shift 16,
maximum fast-, middle-, and slow-band angles are 16, 0.1386, and
0.001489 radians. In the slow band, both predictors reach MAEs of
$4$--$6\times10^{-6}$ at native precision. In the middle band, the
Jacobian already removes 94.0--94.8\% of zero-predictor error; finite
MAEs of $4.0$--$6.4\times10^{-5}$ further reduce the small remaining
error. Large fast-band shifts are deliberately outside the tangent
regime. At middle-band shift 128, finite MAE is 0.002112--0.003543
against Jacobian MAE 0.016428--0.027925, for effects of mean magnitude
0.027545--0.038921. Full strata are retained in
Appendix~\ref{app:bandstrata}. The angular strata identify where
finite corrections have a material absolute effect.

The predicted structural distinctions are also measurable. Fast--middle
interaction contributes 3.33--9.20\% of joint local-response norm at
shift 16 and 16.20--26.94\% at shift 128. Under actual prompt
reordering, direct phase accounts for essentially all of the layer-0
response, whereas changed contextual features have signed projections
of 0.935--0.976 onto the total response at layers 14 and 21. Thus
frozen-feature phase edits and prompt reorderings answer different
experimental questions.

\subsection{The joint interaction improves downstream prediction}
The joint-cache study fixes six settings before evaluation:
Qwen2.5-0.5B-Instruct at layer 18 and Qwen2.5-1.5B-Instruct at
layers 14 and 21, each on retrieval and SST-2. Each setting uses 32
prompt pairs and eight spans, with no filtering by accuracy or outcome.
Independent key-only, value-only, and joint donor replacements supply
4,608 executions; two additional joint-edit strengths at 1.5B
retrieval layer 21 bring the total to 5,120.

\begin{table}[t]
\centering\small
\caption{\textbf{Joint-edit prediction across all six settings.}
Downstream margin MAE over 32 prompt pairs and eight joint edits each.
Separate uses the exact finite-key response plus baseline-weighted
values. Quadratic adds the leading mixed term; exact adds $C_{KV}$.
Complete controls and paired intervals are in
Tables~\ref{tab:jointcoverage} and~\ref{tab:jointpaired}.}
\label{main:jointtable}
\begin{tabular}{lllrrr}
\toprule
Model & Task & Layer & Separate & Quadratic & Exact\\
\midrule
0.5B & Retrieval & 18 & 0.001042 & $2.98{\times}10^{-4}$ & $4.37{\times}10^{-6}$\\
0.5B & SST-2 & 18 & 0.003099 & $5.99{\times}10^{-4}$ & $2.64{\times}10^{-5}$\\
1.5B & Retrieval & 14 & 0.003342 & 0.001525 & $8.22{\times}10^{-5}$\\
1.5B & Retrieval & 21 & 0.107864 & 0.087683 & 0.007844\\
1.5B & SST-2 & 14 & 0.003194 & 0.001418 & $3.42{\times}10^{-4}$\\
1.5B & SST-2 & 21 & 0.109616 & 0.080700 & 0.004510\\
\bottomrule
\end{tabular}
\end{table}

Retaining $C_{KV}$ lowers error by 89.3--99.6\% against separate
attribution, a reduction of more than a factor of nine in every setting,
and by 75.9--98.5\% against the quadratic correction. All individual
paired 95\% intervals favour exact joint. Its contraction correlates
0.9970--0.99999 with the independently executed downstream four-cell
interaction. The four-cell target is measured after executing the remaining
network; the comparison tests whether the isolated local interaction
accounts for that downstream effect.

At strengths $\alpha=0.1,0.5,1$, exact-joint MAEs are
0.000173, 0.002330, and 0.007844; quadratic MAEs are 0.000237,
0.013071, and 0.087683. At the smallest strength, both recover the
two most influential spans perfectly. The gap grows with edit strength,
while the exact predictor's own error also grows
(Figure~\ref{main:empirical}). Appendix~\ref{sec:jointexperiment}
gives the complete execution protocol and additional plots.

\begin{figure}[t]
\centering
\includegraphics[width=\linewidth]{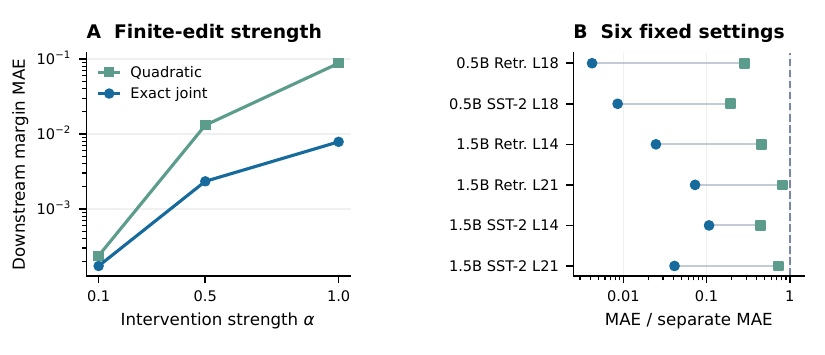}
\caption{\textbf{Retaining the joint interaction improves executed-effect
prediction.} Left: downstream MAE at the three measured edit strengths
on 1.5B retrieval, layer 21; lines connect measured means on a log scale.
Right: the six full-replacement settings, with each MAE divided by the
separate predictor's MAE in that setting. Absolute values are in
Table~\ref{main:jointtable}. Predictors share a baseline gradient.
Points show reported means; paired 95\% intervals are reported in
Table~\ref{tab:jointpaired}.}
\label{main:empirical}
\end{figure}

\subsection{Ranking and repair decisions}
Deletion-effect ranking attains Spearman correlations of 0.950--0.989
against executed effects. The matched single-entry precision-restoration
comparison reduces mean relative RMSE by 4.86\%, 8.13\%, and 6.67\%
in three task/precision conditions, with paired difference intervals
below zero; across all six, the point estimate improves in five and ties
in one. Four-entry greedy restoration improves 19.8--54.9\% against
fixed second-order selection. This evaluates the complete greedy procedure,
including re-scoring after each restoration
(Appendix~\ref{sec:positionalapplications}).

The accompanying instrument also records a selected synthetic API repair
on Qwen2.5-0.5B-Instruct. Ordinary generation passes 0/6 fixed checks,
source emphasis passes 1/6, and an installed no-op reproduces the 0/6
failure. A ranked attention edit produces a program that passes 6/6.
This example illustrates the score--execute--check workflow; its adaptive
selection and supplied ranking target are documented in
Appendix~\ref{app:answerrepair}.

\subsection{Preparation, reuse, and numerical agreement}
The sparse and batched-dense timings request the same scalar local
response after contraction with a fixed gradient. On the native A100
case with 112 tokens and eight spans, prepared sparse scoring takes
0.336\,ms versus 0.925\,ms. Including preparation, one strength takes
2.521\,ms versus 2.083\,ms; three strengths sharing preparation take
3.292\,ms versus 4.014\,ms. Across seven workloads, preparation-inclusive
reuse at three strengths reduces latency by 18.0--21.8\%, while a
single strength is 18.2--21.0\% slower. These are local evaluation
costs with gradient acquisition excluded, not full-inference speedups
or comparisons with fused kernels (Appendix~\ref{sec:jointexperiment}).

Native reconstruction checks the effective-weight identity with affine
biases and grouped-query attention. The joint sweep's largest FP64
sparse--dense discrepancy is $1.31\times10^{-14}$; its largest local
contraction discrepancy against FP32 execution is $2.60\times10^{-6}$.
The positional task workers record 1,166,880 passing numerical checks.
These verify implementation agreement independently of downstream accuracy.

\section{Related work}
\label{main:related}
RoPE supplies the relative-position rotation \citep{su2021}, and
$\ph$ is classical in exponential integration \citep{hochbruck2010}.
We connect these ingredients through finite attention responses, their
rotary/softmax error decomposition, and operational intervention rules.
The positive quadratic curvature recovers the attention-preserving query
subspace; affine log odds then give a minimum-norm control construction.
The effective-weight result characterises arbitrary
affine projections with softmax intact through a query-conditioned
quadratic, complementing linear-attention constructions for a specified
least-squares learner \citep{vonoswald2023}, fast weights
\citep{schlag2021}, implicit updates \citep{dai2023}, and Hopfield
correspondences \citep{ramsauer2021}. Appendix~\ref{sec:regressionpilot}
separately compares trained predictions with an independent regression learner.

Attribution patching supplies the shared-gradient scoring framework
\citep{syed2024}. Relative to AtP$^*$'s finite QK correction at fixed
values \citep{kramar2024}, our joint decomposition identifies the additional
interaction when keys and values change together. The separate comparator
retains that finite-key correction; the quadratic comparator also includes
the leading mixed term. PASTA steers attention toward user-specified spans
\citep{zhangpasta2024}; here the emphasis is on explicit finite responses,
feasible ratio constraints, candidate scoring, and executed verification.
Our local distortion criteria also provide components for cache-selection
methods based on observed attention patterns, such as
SnapKV \citep{li2024snapkv}.
We specify intervention endpoints, controls, metrics, and statistical units
following the methodological concerns of \citet{zhangnanda2024}.

\section{Discussion and limitations}
\label{main:discussion}
The value of the compact identity is the set of consequences it exposes:
an explicit interaction for attribution, a geometry for constructing
attention edits, and local error criteria for cache decisions and
approximation checks. These operations require neither a learned surrogate
response nor a small-edit expansion. The positional and joint-KV studies
establish that retaining the finite effects improves downstream prediction
in the evaluated settings; frequency and strength sweeps identify where
the correction matters. This makes the calculus a building block that
can be incorporated into other interpretability and inference methods.

The positional construction assumes fixed frequencies and frozen content
at the edited readout; actual prompt reordering also changes contextual
features. Empirical coverage remains within the Qwen family, and answer
prediction uses a frozen downstream gradient. The KL certificate's
empirical tightness is unmeasured. Multi-entry restoration lacks a matched
second-order re-scoring control, and the API repair is an adaptively
selected demonstration. Total intervention-search cost, including gradient
acquisition and executed verification, remains unmeasured.

The resulting instrument makes finite interventions explicit and
reviewable: specify the attention relationship or edit family, compute
the local response, rank candidates against a stated objective, and check
selected edits by execution. The gradient-step representation and
gauge-invariant query features characterise the map being controlled.
\label{main:end}
\clearpage

\section*{Reproducibility statement}
The appendices preserve the complete derivations, model revisions,
prompt construction, calibration policy, numerical tolerances, and
prompt-paired evaluation procedures from the full study.
Appendix~\ref{app:jointsupplement} gives the fixed joint-KV sweep and
full tables; Appendix~\ref{app:jointprotocol} describes the
standalone runner, native preflights, and replay captures.
The source includes a figure-generation script and the reported summary
values used in the empirical plots; the tables retain the paired intervals.
Experimental code, saved result archives, and the instrument
are separate research artifacts supporting the preprint.

\section*{AI use statement}
Generative AI tools assisted with experimental design and interpretation,
software implementation and debugging, mathematical exposition and
derivation checks, manuscript organisation and editing, and preparation
of explanatory figures. The reported computations use numerical
preflights, native execution controls, and saved measurements as
described in the experimental appendices. The authors are responsible
for the final claims, evidence, code, and manuscript.

\clearpage
\appendix
\counterwithin{equation}{section}
\counterwithin{table}{section}
\counterwithin{figure}{section}
\counterwithin{theorem}{section}
\section{RoPE derivatives and finite differences}
\label{sec:positionaloperators}
For fixed content $z_0$, let $z(p)=R(p)z_0$ be a rotated column feature. The generator
in \eqref{main:rotation} gives
\begin{equation}
\partial_p^n z(p)=A^n z(p),\qquad n\ge1.
\label{eq:derivatives}
\end{equation}
The matrix extension of \eqref{eq:rho} is defined by the convergent power series
$\rho(B)=\sum_{n\ge0}B^n/(n+1)!$. Consequently
\begin{equation}
\boxed{\quad e^{\delta A}-I=\delta A\rho(\delta A),\qquad
z(p+\delta)-z(p)=\delta A\rho(\delta A)z(p).\quad}
\label{eq:matrixrho}
\end{equation}
The power-series definition also covers zero-frequency blocks. These identities expose
RoPE as a bank of position-response operators on frozen features: $A$ gives the
infinitesimal response and $A\rho(\delta A)$ gives the exact finite-difference quotient.
The same entire function converts the scalar softmax increment $e^s-1$ into
$s\rho(s)$ and the matrix rotary increment $e^{\delta A}-I$ into
$\delta A\rho(\delta A)$. In \eqref{eq:original}, rotary position structure enters
through $R(p_j-p_i)$ and its scores, while $\rho(s_{ij})$ retains their full exponential
response. The derivatives in \eqref{eq:derivatives} are with respect to position at fixed content.

\begin{figure}[htbp]
\centering
\includegraphics[width=\linewidth]{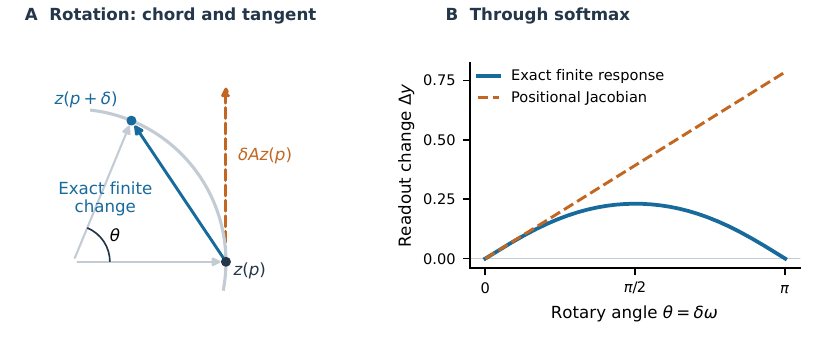}
\caption{\textbf{The derivative determines the tangent; its integral gives
the finite response.} Left: a single rotary plane, with the exact chord
and the tangent step at the same displacement. Right: an analytic two-key
example with scores $(\sin\theta,0)$ and scalar values $(1,0)$.
The exact readout change is
$\operatorname{sigmoid}(\sin\theta)-1/2$; the positional Jacobian gives
$\theta/4$. This illustration separates the mathematical issue from
downstream model behaviour; it is not an empirical model measurement.}
\label{fig:ropegeometry}
\end{figure}

\section{Exact effective weights with softmax intact}
\label{sec:effectiveweights}
\subsection{RoPE, affine projections, and the attended bank}
Let $x_i\in\R^m$ be a column input, with query/key width $d$ and value width $r$.
Use row-vector projections
\begin{equation}
q_i=x_i\T W_q+b_q,\qquad
k_j=x_j\T W_k+b_k,\qquad
v_j=x_j\T W_v+b_v,
\label{eq:projections}
\end{equation}
where $W_q,W_k\in\R^{m\times d}$, $W_v\in\R^{m\times r}$, and biases are rows.
For standard RoPE with fixed frequencies \cite{su2021}, set $\tau=\sqrt d$ and
\begin{equation}
R(p)=e^{pA},\qquad
A=\operatorname{diag}_{\ell=1}^{d/2}
\begin{pmatrix}0&-\omega_\ell\\ \omega_\ell&0\end{pmatrix},\qquad
R(p_i)\T R(p_j)=R(p_j-p_i).
\label{eq:rotation}
\end{equation}
The rotated rows are $\widetilde q_i=q_iR(p_i)\T$ and
$\widetilde k_j=k_jR(p_j)\T$. Augment the query input and projection by
\begin{equation}
u_i=\begin{pmatrix}x_i\\1\end{pmatrix},\qquad
\widehat W_q=\begin{pmatrix}W_q\\b_q\end{pmatrix},\qquad
L_i=\frac{\widehat W_qR(p_i)\T}{\tau},\qquad
a_{ij}=L_i\widetilde k_j\T.
\label{eq:features}
\end{equation}
Then the raw score is exactly $s_{ij}=\widetilde q_i\widetilde k_j\T/\tau=u_i\T a_{ij}$.
In the bias-free case this is
$s_{ij}=x_i\T W_qR(p_j-p_i)W_k\T x_j/\tau$.
Let $\Iset_i$ be a finite nonempty set of permitted keys; causal masking selects this
set. For deterministic attention, define
\begin{equation}
N_i=|\Iset_i|,\qquad Z_i=\sum_{j\in\Iset_i}e^{s_{ij}},\qquad
\mu_i=\frac1{N_i}\sum_{j\in\Iset_i}v_j,\qquad
y_i=\frac1{Z_i}\sum_{j\in\Iset_i}e^{s_{ij}}v_j.
\label{eq:attention}
\end{equation}
Here $\mu_i$ is the uniform arithmetic mean of the permitted values.

\subsection{The exponential divided difference}
The entire function
\begin{equation}
\rho(s)=\varphi_1(s)=\exp[0,s]
=\int_0^1e^{ts}\,\dd t
=\sum_{n=0}^{\infty}\frac{s^n}{(n+1)!}
=\begin{cases}(e^s-1)/s,&s\ne0,\\1,&s=0\end{cases}
\label{eq:rho}
\end{equation}
is the first divided difference of the exponential at $0$ and $s$, with its continuous
value at coincident arguments. It is the standard $\varphi_1$ of exponential
integration \cite{hochbruck2010}; $\rho(s)>0$ for real $s$ and $e^s-1=s\rho(s)$.

\begin{theorem}[Exact effective-weight representation]
\label{thm:effectiveappendix}
For the head in \eqref{eq:projections}--\eqref{eq:attention}, define
\begin{equation}
\boxed{\quad c_{ij}=\frac{\rho(s_{ij})}{Z_i},\qquad
\DM_i=\sum_{j\in\Iset_i}c_{ij}a_{ij}(v_j-\mu_i),\qquad
y_i=\mu_i+u_i\T\DM_i.\quad}
\label{eq:identity}
\end{equation}
The last equality holds exactly for every query and all finite projection weights
and inputs.
\end{theorem}
\begin{proof}
The centred values sum to zero. Therefore
\begin{align*}
y_i-\mu_i
&=Z_i^{-1}\sum_j e^{s_{ij}}(v_j-\mu_i)
 =Z_i^{-1}\sum_j(e^{s_{ij}}-1)(v_j-\mu_i)\\
&=Z_i^{-1}\sum_j s_{ij}\rho(s_{ij})(v_j-\mu_i)
 =u_i\T\sum_j c_{ij}a_{ij}(v_j-\mu_i).
\end{align*}
\end{proof}
In particular, without projection biases the effective matrix in the original weights is
\begin{equation}
\DM_i=\frac1{\tau Z_i}\sum_{j\in\Iset_i}\rho(s_{ij})
\bigl[W_qR(p_j-p_i)W_k\T x_j\bigr]\bigl[x_j\T W_v-\mu_i\bigr].
\label{eq:original}
\end{equation}
At a fixed bank and query position, write $L=L_i$ and $a_j=a_{ij}$. Then
$\DM_i=LT_i$, where
$T_i=\sum_j c_{ij}\widetilde k_j\T(v_j-\mu)\in\R^{d\times r}$.
The cache supplies the outer-product directions; the query selects their positive
coefficients. This gives $\operatorname{rank}(\DM_i)\le d$ in this representation.

\begin{corollary}[A query-conditioned unit gradient step]
For each fixed query, define $B\in\R^{(m+1)\times r}$ and
\begin{equation}
\mathcal E_i(B)=\frac12\sum_{j\in\Iset_i}c_{ij}
\norm{a_{ij}\T B-(v_j-\mu_i)}_2^2.
\quad\text{Then}\quad
0-\nabla_B\mathcal E_i(0)=\DM_i.
\label{eq:gradient}
\end{equation}
\end{corollary}
\begin{proof}
With the query, features, centred values, and coefficients held fixed,
$\nabla_B\mathcal E_i(B)=\sum_j c_{ij}a_{ij}(a_{ij}\T B-(v_j-\mu_i))$;
evaluation at $B=0$ gives the result.
\end{proof}
A unit step from zero on $\mathcal E_i$ gives the effective matrix and its exact
readout $y_i=\mu_i+u_i\T\DM_i$. The regression features are $a_{ij}$ and the
responses are the centred cached values $v_j-\mu_i$; every coefficient is computed
directly from the query and its bank through \eqref{eq:identity}. The step therefore
identifies a matrix in input coordinates and the tokenwise outer products that
form it. At a fixed bank and position, changing the query reweights these same
directions. Section~\ref{sec:effectivemaps} makes that change explicit through the
score-shift family and the exact residual from reusing an anchor matrix.

As all scores approach zero, $c_{ij}\to1/N_i$; the first-order readout is
$\mu_i+N_i^{-1}\sum_j s_{ij}(v_j-\mu_i)$. This is the centred outer-product form
underlying linear-attention gradient constructions. The corresponding regression
step is specified by the feature/label encoding, projection choices, and step scale
of von Oswald et al.\ \cite{vonoswald2023}.

\section{Residual coordinates, gauge, and exact reuse error}
\label{sec:effectivemaps}
\subsection{Residual write and score-shift family}
In the bias-free case, if the head input is also the residual input and
$W_o\in\R^{r\times m}$, then
\begin{equation}
x_i\T+y_iW_o=x_i\T(I_m+\DM_iW_o)+\mu_iW_o.
\label{eq:residual}
\end{equation}
For a pre-normalised head, use $x_i=\operatorname{Norm}(h_i)$ and the exact residual
$h_i\T+(\mu_i+u_i\T\DM_i)W_o+b_o$; projected head writes are summed for multi-head
attention. Equation~\eqref{eq:identity} is a secant representation of the readout.
At a fixed bank and position, its differential is
$\dd y_i=\dd u_i\T\DM_i+u_i\T\dd\DM_i$; the second term contains the
query derivatives of the coefficients.

Effective matrices depend on a score-shift convention. For any fixed
$g\in\R^{m+1}$, put $\sigma_i=u_i\T g$. Softmax is unchanged by subtracting
$\sigma_i$ from every attended score. Applying Theorem~\ref{thm:effectiveappendix} to the shifted features gives
\begin{equation}
\begin{aligned}
a_{ij}^g&=a_{ij}-g,&
c_{ij}^g&=\frac{e^{\sigma_i}\rho(s_{ij}-\sigma_i)}{Z_i},\\
\DM_i^g&=\sum_j c_{ij}^g a_{ij}^g(v_j-\mu_i),&
y_i&=\mu_i+u_i\T\DM_i^g.
\end{aligned}
\label{eq:gauge}
\end{equation}
Key shifts give the subfamily $g=L_i b$, of dimension at most $d$ at a fixed position.
We call \eqref{eq:identity} the \emph{original-score gauge}.
The displayed transformations give a family of exact matrices
for the same output map.

\subsection{Global affine representability}
\begin{proposition}[Characterisation of global affine readouts]
Fix a finite nonempty key/value bank and query position, and let
$F:\R^m\to\R^r$ be its softmax head output as the query coordinate ranges over
all of $\R^m$. A fixed affine representation $F(x)=b+x\T M$ exists for every
$x\in\R^m$ if and only if $F$ is constant. In that case $M=0$ and $b=F(0)$.
\label{prop:bounded}
\end{proposition}
\begin{proof}
Every output is a convex combination of the fixed values, so
$\norm{F(x)}_2\le\max_j\norm{v_j}_2$ for all $x$.
If $M\ne0$, choose $z$ with $z\T M\ne0$; then $\norm{b+t z\T M}_2\to\infty$
as $|t|\to\infty$. Hence $M=0$ and $F$ is constant. Conversely, a constant
$F$ has the stated affine representation.
\end{proof}
Matrix transfer on a finite query set is quantified by the following exact residual.

\subsection{The exact term omitted by freezing}
Fix the bank and query position so that $a_j$ and $\mu$ are shared, and select an
anchor query $a$. Define $y_i^{\mathrm{fr}}=\mu+u_i\T\DM_a$. Then
\begin{equation}
\boxed{\quad
e_i=y_i-y_i^{\mathrm{fr}}
=u_i\T(\DM_i-\DM_a)
=\sum_{j\in\Iset}(c_{ij}-c_{aj})(u_i\T a_j)(v_j-\mu).
\quad}
\label{eq:tokenerror}
\end{equation}
Thus $y_i=\mu+u_i\T\DM_a+e_i$, and $e_iW_o$ is the exact projected correction.
Each summand resolves this correction by cached token. In bias-free coordinates,
freezing anchor $0$ omits exactly $x_i\T(\DM_i-\DM_0)$, with
\begin{equation}
\norm{e_i}_2\le\norm{u_i}_2\norm{\DM_i-\DM_a}_{\mathrm{op}}.
\label{eq:bound}
\end{equation}
The fixed-position condition makes the features shared; changing the position also
changes $L_i$, and changing the bank can change both features and mean.

\subsection{Positive curvature and attention-preserving edits}
\label{app:curvature}
Fix the bank, mask, query position, and affine query projection, and
choose a reference key $j_0$. Suppress the query index and use the
score-shift gauge $g=a_{j_0}$ from \eqref{eq:gauge}. Write
\begin{equation}
 \begin{gathered}
 \widetilde a_j=a_j-a_{j_0},\qquad
 \widetilde s_j=u\T\widetilde a_j,\qquad
 \widetilde c_j=\frac{\rho(\widetilde s_j)}{\sum_k e^{\widetilde s_k}}>0,\\
 \mathcal E_u^{j_0}(B)=\frac12\sum_j\widetilde c_j
       \norm{\widetilde a_j\T B-(v_j-\mu)}_2^2,\qquad
 G_u=\sum_j\widetilde c_j\widetilde a_j\widetilde a_j\T.
 \end{gathered}
 \label{eq:curvaturegauge}
\end{equation}
Theorem~\ref{thm:effectiveappendix} applies in this gauge, giving
$y-\mu=u\T[-\nabla_B\mathcal E_u^{j_0}(0)]$.
With the coefficients frozen during differentiation, the Hessian acts
on a matrix perturbation $X$ as $X\mapsto G_uX$.
Let $P=(I_m,0)\T$ embed an input edit $h$ as $(h\T,0)\T$, so that
the homogeneous bias coordinate is held fixed. Let $H_{\rm all}$ have
rows $\widetilde a_j\T P$ over all attended keys, including the zero
reference row, and define $\Gamma_u=P\T G_uP$.

\begin{corollary}[Curvature characterises attention-preserving edits]
\label{cor:curvaturecontrol}
Under the fixed-bank assumptions above,
\begin{equation}
 \ker\Gamma_u=\ker H_{\rm all}
 =\{h:p(x+h)=p(x)\}.
 \label{eq:curvaturekernel}
\end{equation}
This kernel is independent of $x$ and of the chosen reference key.
\end{corollary}
\begin{proof}
Positivity gives
\[
 h\T\Gamma_uh
 =\sum_j\widetilde c_j(\widetilde a_j\T Ph)^2,
\]
which vanishes exactly when every relative score is unchanged. Since
$\log(p_j/p_{j_0})=u\T\widetilde a_j$, these equalities are equivalent
to preserving every probability ratio, and hence the positive normalised
distribution. The rows of $H_{\rm all}$ span all pairwise feature
differences, independently of the reference. The positive weights may
change with $x$, but their null space does not.
\end{proof}
Thus the quadratic exposes an exact distribution-preserving subspace in
the editable coordinates. The next construction uses the same feature
differences to test feasibility and produce the smallest requested edit.
Values can introduce additional output cancellations: preserving the
head output is a weaker condition than preserving its attention distribution.

\subsection{Gauge-invariant query control}
\label{app:querycontrol}
Fix the key/value bank, mask, and affine query projection at one head
readout. Apply any prescribed positional key edit first, obtaining
$a'_j=((\bar a'_j)\T,\beta'_j)\T$. For a query input $x$ and edit $h$,
\begin{equation}
 \log\frac{p_j(x+h)}{p_k(x+h)}
 =u\T(a'_j-a'_k)+(\bar a'_j-\bar a'_k)\T h .
 \label{eq:querylogodds}
\end{equation}
All probabilities here refer to the same edited bank. Softmax's common
denominator cancels in the ratio, so this equation is affine and exact
for arbitrary finite $h$. Form $H$ and $b$ as in
Section~\ref{main:instrument}. The requested log ratios are achievable
if and only if $b\in\operatorname{range}(H)$, equivalently
$(I-HH^\dagger)b=0$. In that case every solution is
\[
 h=H^\dagger b+n,\qquad n\in\ker H.
\]
The two terms are orthogonal, giving
$\norm h_2^2=\norm{H^\dagger b}_2^2+\norm n_2^2$ and proving
the unique minimum-norm claim. With all attended key pairs included,
preserving their ratios is equivalent to preserving the entire positive,
normalised attention distribution. Thus $\ker H$ gives its invisible
query directions and the row space of $H$ gives its controllable
query subspace. A common feature shift $a'_j\mapsto a'_j-g$ leaves
both $H$ and $b$ unchanged, so this construction is invariant to the
score-shift gauge. In contrast, $\DM_i$ is a query-dependent secant
matrix; its kernel alone need not characterise attention-preserving edits.

Minimum norm refers to the affine head-input coordinate $x$, with the
bank frozen. Realisation through an upstream normalisation or other
constrained parameterisation requires the corresponding constrained solve.
For inconsistent constraints, $H^\dagger b$ is only a least-squares
solution: its nonzero residual must be reported rather than treating the
requested ratios as achieved. Numerical implementations use an SVD rank
tolerance and verify the achieved ratios against recomputed probabilities.

\section{Finite changes and cache attribution}
\label{sec:finitechanges}
At a fixed query, a cache edit changes the effective-matrix readout and may change
the uniform value mean. Their combined change is the exact attention increment.
We express it in cache coordinates to separate key, value, and interaction effects
and obtain the local response used for downstream attribution.

\subsection{Exact score and value increments}
Fix one query and its attended set, and suppress the query index. Let
$p_j=e^{s_j}/Z$, $y=\sum_jp_jv_j$, and consider changed scores and values
$s'_j=s_j+\delta_j$, $v'_j=v_j+\varepsilon_j$. Define
\begin{equation}
 D=\sum_jp_je^{\delta_j},\qquad
 t_j=\delta_j-\log D,\qquad p'_j=p_je^{t_j}.
 \label{eq:incrementweights}
\end{equation}
The same exponential divided difference as in \eqref{eq:identity} gives
\begin{equation}
\boxed{\quad
 y'-y=\sum_jp_jt_j\rho(t_j)(v_j-y)
             +\sum_jp'_j\varepsilon_j.\quad}
\label{eq:finiteincrement}
\end{equation}
Indeed, $p'_j-p_j=p_j(e^{t_j}-1)=p_jt_j\rho(t_j)$ and
$\sum_j(p'_j-p_j)=0$. Expanding $\sum_jp'_jv'_j-\sum_jp_jv_j$ then proves
\eqref{eq:finiteincrement}. The first term reallocates attention mass and the second
changes the attended values; their sum includes their interaction exactly.
For numerical evaluation, we compute
$t_j=\log p'_j-\log p_j$ using log-softmax. When $t_j\ge0$,
$p'_j-p_j=-p'_j\operatorname{expm1}(-t_j)$; when $t_j<0$,
$p'_j-p_j=p_j\operatorname{expm1}(t_j)$. This evaluates the divided-difference
increment without forming an unnormalised exponential sum.

For a positional edit with the augmented query $u$, values, and attended set
fixed, the uniform mean is unchanged. Applying \eqref{eq:identity} before and
after the edit gives the direct effective-weight relation
\begin{equation}
\begin{aligned}
 \Delta y&=u\T(\DM'-\DM)
 =\sum_j(p'_j-p_j)(v_j-\mu)\\
 &=\sum_j(p'_j-p_j)(v_j-y)
 =\sum_jp_jt_j\rho(t_j)(v_j-y).
\end{aligned}
\label{eq:matrixincrement}
\end{equation}
The change of centre uses $\sum_j(p'_j-p_j)=0$.
Equation~\eqref{eq:matrixincrement} evaluates the effective-weight change
directly in cache coordinates.

\subsection{Attribution patching and exact joint responses}
\label{sec:jointattribution}
For the fixed query above, define the score-level linear response
\begin{equation}
 L_s=\sum_jp_j\delta_j(v_j-y)
     =\sum_jp_jt_j(v_j-y),\qquad
 V_0=\sum_jp_j\varepsilon_j.
 \label{eq:atplinear}
\end{equation}
The equality follows from $\sum_jp_j(v_j-y)=0$. Contracting $L_s+V_0$
with the baseline downstream gradient, after output projection, gives
score/value attribution patching. For finite rotary edits, the positional
Jacobian additionally replaces $\delta_j$ by $\delta_j^{\mathrm{lin}}$.

Set
\begin{equation}
 K_f=\sum_j(p'_j-p_j)(v_j-y),\qquad
 C_{KV}=\sum_j(p'_j-p_j)\varepsilon_j.
 \label{eq:kvcomponents}
\end{equation}
The fixed-value QK correction uses $K_f$ before the gradient contraction.
Indeed, its probability-gradient contraction can be written
\begin{equation}
 \sum_j(p'_j-p_j)\langle g,v_jW_o\rangle
 =\sum_j(p'_j-p_j)\langle g,(v_j-y)W_o\rangle.
 \label{eq:atpcentered}
\end{equation}
This equality is for the sum; the right-hand side defines the centred token
allocation used here. Equation~\eqref{eq:finiteincrement} now reads
\begin{equation}
 \boxed{\quad y'-y=K_f+V_0+C_{KV}.\quad}
 \label{eq:jointattribution}
\end{equation}
Thus adding a fixed-value key correction and a baseline-weighted value correction
omits exactly $C_{KV}$. In a four-cell intervention on probabilities and values,
this term is
\begin{equation}
 C_{KV}=y(p',v+\varepsilon)-y(p',v)
                   -y(p,v+\varepsilon)+y(p,v).
 \label{eq:kvfactorial}
\end{equation}
It measures the change in a value edit's effect when attention is reallocated.
For joint cache edits in the original effective-weight coordinates, the uniform
mean also changes, giving
\begin{equation}
 y'-y=(\mu'-\mu)+u\T(\DM'-\DM)
      =K_f+V_0+C_{KV}.
 \label{eq:jointmatrix}
\end{equation}
The query is fixed in this equality; the key and value bank changes.
Equation~\eqref{eq:jointmatrix} connects changes in the input-coordinate map to
separate key, value, and interaction contributions evaluated in cache coordinates.

The exact score-linearisation remainder is
\begin{equation}
 R_s=K_f-L_s
 =\sum_jp_jt_j\bigl(\rho(t_j)-1\bigr)(v_j-y).
 \label{eq:rhoremainder}
\end{equation}
Consequently, the joint response decomposes into $L_s+V_0$, the softmax
remainder $R_s$, and the interaction $C_{KV}$. A generator-linearised positional
comparator has the additional remainder
$\sum_jp_j(\delta_j-\delta_j^{\mathrm{lin}})(v_j-y)$.

\begin{samepage}
\begin{proposition}[KL mass of the local remainder]
Let $p,p'$ have positive entries on the fixed attended set and define
\begin{equation}
 w_j=p_j\bigl(e^{t_j}-1-t_j\bigr)
     =p_jt_j\bigl(\rho(t_j)-1\bigr).
 \label{eq:remainderweights}
\end{equation}
Then $w_j\ge0$, $R_s=\sum_jw_j(v_j-y)$, and
\begin{equation}
 \sum_jw_j=\operatorname{KL}(p\Vert p')
          =\log D-\sum_jp_j\delta_j.
 \label{eq:klremainder}
\end{equation}
For a fixed downstream gradient $g$, put
$a_j=\langle g,(v_j-y)W_o\rangle$ and
$b_j=\langle g,\varepsilon_jW_o\rangle$. The error of the score/value linear
response for the projected local proxy obeys
\begin{equation}
\begin{split}
 \bigl|\langle g,[(y'-y)-(L_s+V_0)]W_o\rangle\bigr|
 \le{}&\operatorname{KL}(p\Vert p')\max_j|a_j|\\
 &+\operatorname{TV}(p,p')\bigl(\max_j b_j-\min_j b_j\bigr),
\end{split}
 \label{eq:localcertificate}
\end{equation}
where $\operatorname{TV}(p,p')=\tfrac12\sum_j|p'_j-p_j|$.
\end{proposition}
\begin{proof}
Convexity of the exponential gives $e^t-1-t\ge0$. Since
$\sum_jp_je^{t_j}=1$, the weights sum to
$-\sum_jp_jt_j=\operatorname{KL}(p\Vert p')$.
This bounds the contraction of $R_s$. The positive and negative parts of
$p'-p$ each have mass $\operatorname{TV}(p,p')$; their contractions with $b$
differ by at most that mass times the range of $b$.
\end{proof}
\end{samepage}
For multiple heads, apply the bound to each head with its own output projection
and sum the bounds. The measured answer-margin prediction below retains a
baseline downstream gradient; \eqref{eq:localcertificate} certifies its local
attention component.

The bound also certifies local score orderings. For candidate $c$, let
$S_c=\langle g,\Delta y_cW_o\rangle$ be the exact local score,
$\widehat S_c=\langle g,(L_{s,c}+V_{0,c})W_o\rangle$ its first-order
counterpart, and $B_c$ a bound from \eqref{eq:localcertificate}. Then
\begin{equation}
 \widehat S_a-\widehat S_b>B_a+B_b\quad\Longrightarrow\quad S_a>S_b.
 \label{eq:certifiedlocalranking}
\end{equation}
Indeed, $S_a\ge\widehat S_a-B_a>\widehat S_b+B_b\ge S_b$.
Likewise, $|\widehat S_c|>B_c$ certifies its sign. These statements concern
the local scores under a fixed readout; they do not bound the remaining
downstream nonlinearity. Evaluating the bound is not assumed cheaper than
evaluating the exact response. Hard deletion uses its separate
renormalisation formula below, rather than a finite-logit KL certificate.

\subsection{Sparse evaluation of joint cache edits}
Suppose scores change only on $S$ and values only on $T$. With
$r_j=p_j\delta_j\rho(\delta_j)$ for $j\in S$, the exact joint response is
\begin{equation}
 y'-y=
 \frac{\displaystyle
  \sum_{j\in S}r_j(v_j-y)
  +\sum_{j\in T}p_j\varepsilon_j
  +\sum_{j\in S\cap T}r_j\varepsilon_j}
 {\displaystyle 1+\sum_{j\in S}r_j}.
 \label{eq:sparsejoint}
\end{equation}
Baseline probabilities, the readout, and candidate-group complement statistics
can be shared across interventions. Each additional evaluation uses the changed
entries. In particular,
$\operatorname{KL}(p\Vert p')=\log D-\sum_{j\in S}p_j\delta_j$ and
\begin{equation}
 2\operatorname{TV}(p,p')
 =\sum_{j\in S}p_j\left|e^{\delta_j}/D-1\right|
  +\left(1-\sum_{j\in S}p_j\right)\left|D^{-1}-1\right|.
 \label{eq:sparsetv}
\end{equation}
For evaluation, the denominator is the positive sum
$\sum_{j\notin S}p_j+\sum_{j\in S}p_je^{\delta_j}$, computed in log space
with cached complement mass. Near complete attention concentration, the
corresponding complement-weighted value sum also avoids subtracting two
nearly equal readouts. These formulas support exact joint group evaluation
and local error bounds from the same baseline capture.

\subsection{From a positional edit to an answer-margin prediction}
For a selected RoPE frequency band $b$, let $A_b$ retain its blocks of $A$ and set
the other blocks to zero. Shifting a rotated key by $\delta$ gives
\begin{equation}
 \widetilde k'_j=\widetilde k_j e^{\delta A_b\T},\qquad
 s'_j-s_j=\frac{\widetilde q\,\delta A_b\rho(\delta A_b)
                         \widetilde k_j\T}{\tau}.
 \label{eq:phasescore}
\end{equation}
Together, \eqref{eq:phasescore} and \eqref{eq:finiteincrement} retain both the
finite rotary displacement and its softmax response. The positional-Jacobian
comparator uses
\begin{equation}
 \Delta y_{\mathrm{Jac}}=\sum_j p_j\delta_j^{\mathrm{lin}}(v_j-y),\qquad
 \delta_j^{\mathrm{lin}}=\frac{\widetilde q\,\delta A_b\widetilde k_j\T}{\tau}
 \label{eq:phasejacobian}
\end{equation}
on the shifted keys, with $\delta_j^{\mathrm{lin}}=0$ on the others.

Let $o_\ell=\sum_hy_{\ell h}W_{o,\ell h}$ be the combined attention write at the
intervention site. For the two answer labels, fix the original model's preferred
label $r_*$ and define the signed logit margin
$m=\operatorname{logit}(r_*)-\operatorname{logit}(1-r_*)$.
With $g_\ell=\nabla_{o_\ell}m$ at the original forward pass, compare
\begin{equation}
 \widehat{\Delta m}_{\mathrm{fin}}=\langle g_\ell,\Delta o_{\ell,\mathrm{fin}}\rangle,
 \qquad
 \widehat{\Delta m}_{\mathrm{Jac}}=\langle g_\ell,\Delta o_{\ell,\mathrm{Jac}}\rangle.
 \label{eq:marginprediction}
\end{equation}
The local increment in the first predictor is exact. Both predictors use the same
baseline gradient to propagate that increment through the remaining network.
The experimental target is the measured margin change $\Delta m=m'-m$ after
executing the edit and continuing the model forward.

\subsection{Additive rotary scores and coupled softmax responses}
Fix a query, its attended values, and a set of keys to be shifted. For a band $b$,
let $d^{b}\in\R^N$ be the exact score displacement in
\eqref{eq:phasescore}, with zero displacement on the unshifted keys. Write
$F(s)=\sum_jp_j(s)v_j$ for the head readout as a function of its scores.

\begin{proposition}[Softmax interaction between disjoint rotary bands]
For disjoint rotary bands $b,c$, their joint score displacement satisfies
$d^{b\cup c}=d^b+d^c$. The corresponding output interaction is
\begin{equation}
\begin{aligned}
 \mathcal H_{bc}
 &=F(s+d^b+d^c)-F(s+d^b)-F(s+d^c)+F(s)\\
 &=\int_0^1\!\int_0^1
 D^2F(s+t d^b+z d^c)[d^b,d^c]\,\dd t\,\dd z.
\end{aligned}
\label{eq:bandinteraction}
\end{equation}
For score directions $a,b\in\R^N$, the mixed derivative is
\begin{equation}
 D^2F(s)[a,b]
 =\mathbb E_{p(s)}\!\left[
 (a-\mathbb E_{p(s)}a)(b-\mathbb E_{p(s)}b)(v-F(s))\right].
 \label{eq:softmaxmixed}
\end{equation}
\end{proposition}
\begin{proof}
Disjoint block supports give
$e^{\delta(A_b+A_c)}-I=(e^{\delta A_b}-I)+(e^{\delta A_c}-I)$.
Taking the query--key dot product proves score additivity. Applying the
fundamental theorem of calculus in each score direction proves
\eqref{eq:bandinteraction}. Differentiating
$\partial p_j/\partial s_k=p_j(\mathbf 1_{j=k}-p_k)$ twice through
$F$ yields \eqref{eq:softmaxmixed}.
\end{proof}

All four readouts retain the exact rotations. Thus
\eqref{eq:bandinteraction} measures finite coupling through softmax between
independent rotary score components. Output projection and summing heads
preserve this decomposition by linearity.

\subsection{Exact selection of cache entries for precision repair}
For a single-token restoration, let $\delta_j$ be the score change from restoring
its key and $\varepsilon_j$ the change from restoring its value. With all other
entries held fixed, \eqref{eq:finiteincrement} reduces to
\begin{equation}
 \Delta_j y=\frac{p_j}{1+p_j\delta_j\rho(\delta_j)}
 \left[\delta_j\rho(\delta_j)(v_j-y)+e^{\delta_j}\varepsilon_j\right].
 \label{eq:singlerepair}
\end{equation}
An equivalent implementation computes the new mass
$p'_j=\operatorname{sigmoid}(\log p_j-\log(1-p_j)+\delta_j)$ and the complement
average $y_{-j}=\sum_{k\ne j}p_kv_k/(1-p_j)$, giving
$\Delta_jy=(p'_j-p_j)(v_j-y_{-j})+p'_j\varepsilon_j$.
For dominant tokens, the complement mass and average are recomputed with
log-sum-exp over the other tokens.

Let $e=o_{\mathrm{corrupt}}-o_{\mathrm{original}}$ be the projected-write error.
For candidate $j$, sum its head contributions before evaluating
\begin{equation}
 B_j=\norm{e}_2^2-\norm{e+\Delta_j o}_2^2,\qquad
 j_*\in\operatorname*{arg\,max}_{j\ \mathrm{not\ yet\ restored}} B_j.
 \label{eq:repairbenefit}
\end{equation}
The exact greedy procedure restores $j_*$ and recomputes all candidate benefits
at the new cache state. A restoration replaces the key and value for that physical
token across all grouped-query attention KV groups. The original cache supplies
the restoration values and the target write.

\section{Pretrained effective-weight reconstruction}
\label{sec:reconstruction}
We evaluate Qwen2.5-0.5B \cite{qwen2024}, layer 23 (zero-indexed), on $C=16$ article
prefixes of 511 tokens, all 14 query heads, and $Q=32$ alternative one-token
continuations at position 511. Continuation $0$ is the observed next token; the other
31 are the model's highest-ranked distinct nonspecial alternatives. Each continuation
reads the same prefix bank. Projection biases and the model's pre-normalised head
inputs are included. Reconstruction is checked both for the prefix-only readout and
for the full causal readout including the continuation's own key/value.

For the latter, if the self token receives mass $\pi_i$, the exact relation is
$y_i^{\mathrm{full}}=(1-\pi_i)y_i^{\mathrm{prefix}}+\pi_i v_i^{\mathrm{self}}$.
The full-bank version of \eqref{eq:identity} uses its own scores and uniform mean.
The 240 reconstruction rows comprise 224 individual head/context pairs and 16
combined projected writes. All pass the reported checks. Combined prefix and full
reconstruction MSEs in float64 are approximately $9.8\times10^{-30}$ and
$9.6\times10^{-30}$. Against the native FP32 combined write, the MSE is
$7.01\times10^{-14}$ and the maximum absolute difference is $1.4\times10^{-5}$.
The float64 figures measure agreement between algebraically equivalent evaluations;
the native comparison measures agreement at the model's execution precision.

The exact tokenwise correction in \eqref{eq:tokenerror} also expresses the
change between a reference effective matrix and the matrix for a new query.
Adding that correction to the reference readout recovers the exact attention
output. Together, the projection-coordinate and cache-coordinate evaluations
verify the construction on the pretrained architecture.

\section{Joint key--value attribution in pretrained attention}
\label{sec:jointexperiment}
We test whether the exact local response predicts executed downstream changes
when contracted with one baseline gradient. The fixed supplemental sweep covers
six model/task/layer settings and 5,120 executed cache interventions. It also
measures the interaction in \eqref{eq:jointmatrix}, sensitivity to edit strength,
and the cost of sparse evaluation. Appendix~\ref{app:jointinitial} retains the
initial single-layer study.

\subsection{Paired prompts and executed cache interventions}
We use frozen Qwen2.5-0.5B-Instruct at zero-indexed layer 18 and
Qwen2.5-1.5B-Instruct \cite{qwen15b2024} at layers 14 and 21. Each model/layer
combination is evaluated on entity--label retrieval and SST-2 sentiment prompts
\cite{socher2013}. All six settings, 32 prompt pairs per task, and eight candidate
spans per prompt are fixed before model evaluation. The same task prompts are
used across the selected model/layer combinations, giving 64 distinct prompt
pairs and 192 setting--pair evaluations. No prompt is filtered by model accuracy
or intervention outcome.

A retrieval donor cyclically permutes entity identifiers and flips their A/B
labels. An SST-2 prompt contains eight balanced labelled training examples and
a held-out validation query; its donor flips the demonstration labels while
preserving the review text and query. Native tokenisation checks equal endpoint
lengths, aligned intervention spans, and unchanged token IDs outside those spans.
Retrieval prompts have 112 tokens; SST-2 prompts have 204--300 tokens, including
the native chat template. Appendix~\ref{app:jointsupplement} records the pinned
checkpoints, prompt construction, and execution protocol.

For every span, we execute donor key-only, value-only, and joint replacements
across all physical KV groups at the selected layer. Each continuation receives
an independent copy of the baseline prefix cache. The baseline query, its self
entry, and all unpatched entries are retained. Four patched continuations share
a batch with one unpatched control. The target is the fixed answer margin
$m=\operatorname{logit}(A)-\operatorname{logit}(B)$; the measured effect subtracts
the control margin from the patched margin in the same batch. Full replacement
gives $6\times32\times8\times3=4{,}608$ patched continuations. At 1.5B retrieval
layer 21, joint edits of strengths $\alpha=0.1$ and $0.5$ add 512 continuations,
for 5,120 in total.

\subsection{Matched predictors from the finite-change decomposition}
At the selected layer $\ell$, let $\mathcal O$ project and sum the head outputs. Define the common
linear readout $\mathcal G(a)=\langle g_{\ell},\mathcal O(a)\rangle$, using one
baseline gradient for all candidates and predictors. Equation~\eqref{eq:jointmatrix}
then supplies the proposed prediction
\begin{equation}
 \widehat{\Delta m}_{\mathrm{joint}}
 =\mathcal G\!\left((\mu'-\mu)+u\T(\DM'-\DM)\right)
 =\mathcal G(K_f+V_0+C_{KV}).
 \label{eq:jointpredictor}
\end{equation}
The change in the uniform mean is included. Computation uses the factored
cache-coordinate response in \eqref{eq:sparsejoint}, with positive group and
complement masses evaluated in log space.

The score/value first-order predictor uses $\mathcal G(L_s+V_0)$. The separate
key/value predictor uses $\mathcal G(K_f+V_0)$, combining the exact grouped
fixed-value key response with the value response at baseline probabilities.
This uses the fixed-value attention correction associated with AtP$^*$
\cite{kramar2024}. A quadratic interaction comparison retains the same exact
key response and adds
\begin{equation}
 C_2=\sum_jp_j\left(\delta_j-\sum_kp_k\delta_k\right)\varepsilon_j,
 \qquad
 \widehat{\Delta m}_{\mathrm{quad}}=\mathcal G(K_f+V_0+C_2).
 \label{eq:quadraticinteraction}
\end{equation}
Along the joint path $(\alpha\delta,\alpha\varepsilon)$,
$C_{KV}(\alpha)=\alpha^2 C_2+O(\alpha^3)$. Thus this comparison retains the
leading mixed term. All predictors evaluate the same complete span and share
the same gradient; none fits a parameter. Direct dense attention recomputation,
followed by $\mathcal G$, supplies an independent equality control for the
factored joint response.

\subsection{Prediction accuracy and influential-span recovery}
For prompt pair $u$ and candidate span $j$, write
$e_{uj}^{a}=|\widehat{\Delta m}_{uj}^{a}-\Delta m_{uj}|$. For each setting, the separate-predictor contrast is
\begin{equation}
 \Delta_{\mathrm{MAE}}
 =\frac1{32}\sum_{u=1}^{32}
   \left[\frac18\sum_{j=1}^{8}
   \left(e_{uj}^{\mathrm{joint}}-e_{uj}^{\mathrm{separate}}\right)\right].
 \label{eq:jointprimary}
\end{equation}
We form its percentile 95\% interval using 5,000 paired bootstrap resamples of
the 32 prompt pairs, keeping all eight joint patches together. Top-two recall
is the overlap between the two spans ranked highest by absolute predicted
effect and the two ranked highest by absolute executed effect, divided by two
and averaged over prompts. Ties use span order. Sign accuracy compares signed
predictions with nonzero executed margin changes. Sign accuracy and top-two
recall are undefined for the zero predictor. We compute the same paired MAE
contrast against the zero, first-order, and quadratic predictors. All intervals
are individual 95\% intervals, without multiplicity adjustment.

Figure~\ref{fig:jointcoverage} shows full-replacement MAE for every setting.
The exact joint predictor has the lowest MAE in all six, reducing error by
75.9--98.5\% against the quadratic interaction correction and by 89.3--99.6\%
against separate key/value attribution. All 24 setting/comparator paired
intervals favor exact joint. The complete means and the separate and quadratic
paired contrasts appear in Tables~\ref{tab:jointcoverage} and
\ref{tab:jointpaired}; the report archive retains all four contrasts.

\begin{figure}[htbp]
\centering
\includegraphics[width=\linewidth]{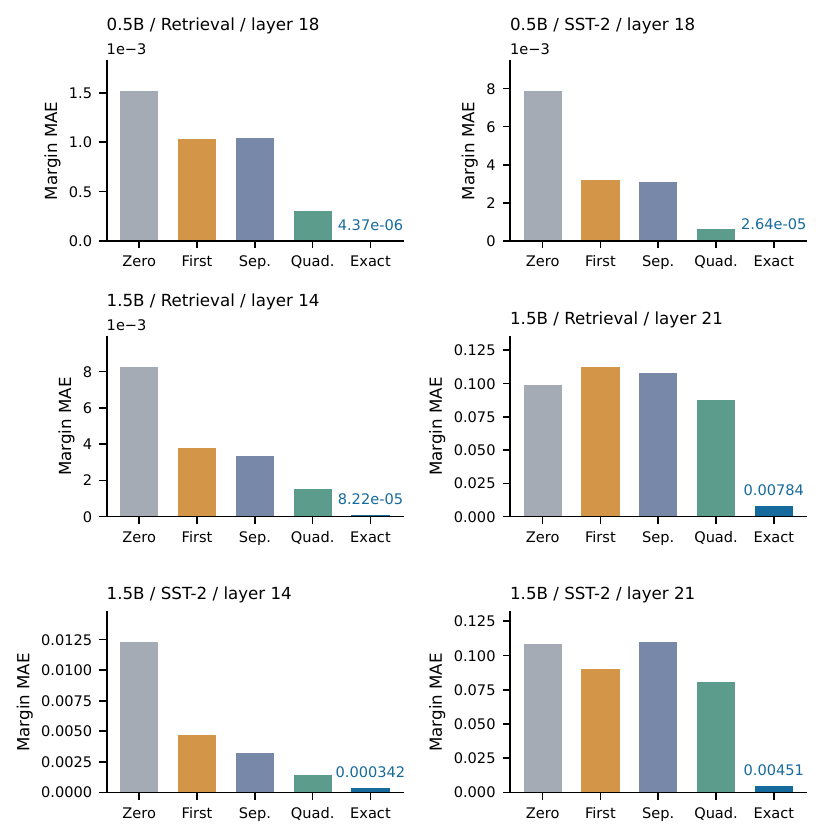}
\caption{Downstream prediction across six fixed settings at full replacement.
Bars plot the reported margin MAEs over 32 prompt pairs with eight joint
span edits each (Table~\ref{tab:jointcoverage}). Axes are scaled separately,
and exact-joint MAEs are labelled. All predictors share the same baseline
downstream gradient; targets are independently executed margin changes.
The quadratic comparator retains the exact key response and leading mixed
term. Paired 95\% intervals are reported in Table~\ref{tab:jointpaired}.}
\label{fig:jointcoverage}
\end{figure}

At 1.5B retrieval layer 21, the exact, separate, and quadratic MAEs are
0.007844, 0.107864, and 0.087683, respectively. At 1.5B SST-2 layer 14, where
the relative gain against the quadratic correction is smallest, they are
0.000342, 0.003194, and 0.001418. Full-replacement top-two recall is
98.44--100\% and sign accuracy is 99.22--100\% for exact joint. The zero
predictor is included throughout: separate attribution exceeds its MAE in both
1.5B layer-21 settings, whereas exact joint improves on zero in every setting.

\subsection{Executed interaction and intervention strength}
Key-only and value-only executions define an additional target:
\begin{equation}
 I_m=(m_{KV}-m_0)-(m_K-m_0)-(m_V-m_0),
 \qquad \widehat I_m=\mathcal G(C_{KV}).
 \label{eq:jointdownstreamfactorial}
\end{equation}
Across the six settings, the signed predicted interaction has Pearson correlation
0.9970--0.99999 with the executed four-cell effect. Interaction MAEs range from
$4.52\times10^{-6}$ to 0.008826, versus $2.96\times10^{-4}$ to 0.092810 for
the quadratic mixed term (Table~\ref{tab:jointfactorial}). The common baseline
gradient propagates each local interaction to a prediction of the independently
executed four-cell effect.

For the strength sweep, rotated keys and values follow
$k'_j=k_j+\alpha(k_j^{\mathrm{donor}}-k_j)$ and
$v'_j=v_j+\alpha(v_j^{\mathrm{donor}}-v_j)$ on the selected span. Prediction and
execution use the same rounded FP32 endpoint tensors. At
$\alpha=0.1,0.5,1.0$, exact-joint MAE is respectively
0.000173, 0.002330, and 0.007844; quadratic MAE is 0.000237, 0.013071,
and 0.087683 (Figure~\ref{fig:jointstrength}). At $\alpha=0.1$, the quadratic
and exact predictors both attain 100\% top-two recall. The small absolute MAE
difference, approximately $-6.34\times10^{-5}$ for exact minus quadratic,
has paired 95\% interval $[-1.09\times10^{-4},-2.88\times10^{-5}]$.
The gap grows with intervention strength, while exact-joint MAE also grows as
the downstream approximation becomes less accurate.

\begin{figure}[htbp]
\centering
\includegraphics[width=\linewidth]{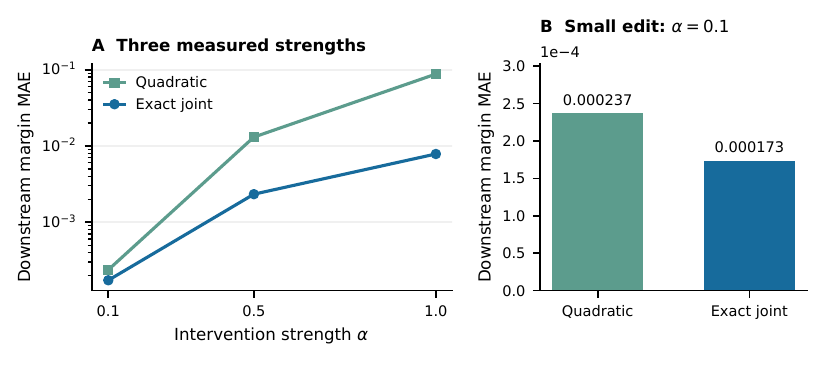}
\caption{Intervention-strength diagnostic on 1.5B retrieval layer 21, using the
same 32 prompt pairs and eight spans at each strength. \textbf{A:} Exact-joint
and quadratic mean downstream errors at the three evaluated strengths;
lines connect measured points on a log scale. \textbf{B:} The same predictors
at $\alpha=0.1$ on a linear scale. The plot uses reported summary means.
The paired 95\% interval for their small-strength difference is
$[-1.09\times10^{-4},-2.88\times10^{-5}]$.}
\label{fig:jointstrength}
\end{figure}

\subsection{Numerical audit and scope}
The supplemental run passes all 14 CPU and 14 CUDA preflight test methods,
including the exact affine-input gradient-step identity, native grouped-query
replay, cache isolation, and finite-difference checks of the projected-write
gradient. Across all 192 setting--pair evaluations, the largest recorded FP64
sparse--dense local discrepancy is $1.31\times10^{-14}$. The maximum difference
between the predicted local contraction and its executed FP32 counterpart is
$2.60\times10^{-6}$. The maximum unpatched control-margin drift is
$1.14\times10^{-5}$; each intervention uses its own batch control. All 5,120
intervention records are present, with no duplicate setting/pair/span/kind/strength
combinations. The smallest downstream MAEs approach the FP32 numerical scale;
they should not be interpreted as exact downstream predictions.

The coverage spans two model sizes within the Qwen family and two prompt tasks.
Baseline A/B accuracy is 46.88\% on retrieval and 59.38\% on SST-2 for 0.5B,
and 96.88\% and 87.50\%, respectively, for 1.5B. These accuracies are reported
as task context; the evaluated quantity is prediction of intervention effects.
The finite-response identity is exact for every permitted head and edit, while
the measured downstream accuracy pertains to these checkpoints, layers, and prompts.

\subsection{Runtime with preparation and reuse}
\label{sec:jointruntime}
We compare sparse evaluation of \eqref{eq:sparsejoint} with batched dense
attention recomputation for the same requested output: one scalar response per
candidate after contraction with the fixed baseline gradient. Both methods
receive identical FP32 inputs and may precontract values with that gradient.
Dense evaluation recomputes the full softmax for every candidate; it batches all
candidates and shares the unchanged value bank. Prefix/model execution and
baseline-gradient acquisition are common inputs and are excluded from both
timings. The measured cost is local response evaluation, rather than full-model
inference latency.

The native timing case uses the first saved 1.5B retrieval layer-21 capture
with 112 tokens and eight seven-token candidate spans. Six synthetic cases
cross $N\in\{128,1024,8192\}$ context tokens with $C\in\{8,256\}$ candidates
and eight-token spans, using 12 query heads, two KV heads, and head width 128.
All cases use an NVIDIA A100-SXM4-40GB, warmed and synchronised timing, and five
rounds with randomised method order. We report medians and interquartile ranges.
The setup-inclusive measurements charge common preparation and, for sparse
evaluation, candidate-group complement preparation. The latter conservatively
scans the bank for each candidate; it is not treated as free preprocessing.

\begin{figure}[htbp]
\centering
\includegraphics[width=\linewidth]{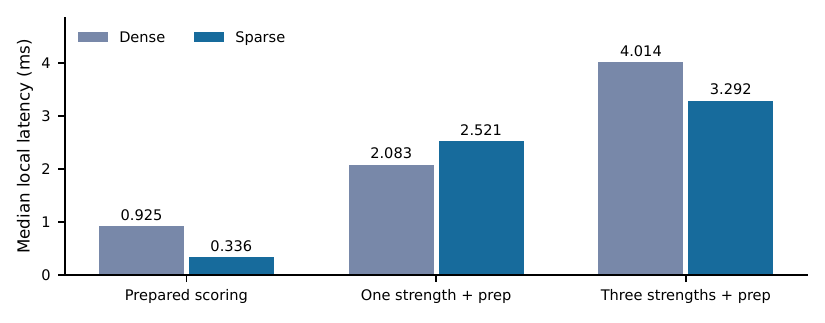}
\caption{Native local-response latency for the same eight candidates on an A100.
Bars show the reported medians over five timing rounds. Scoring-only latency
assumes prepared statistics. The other groups
directly time preparation plus evaluation at one or three strengths; three
strengths share one preparation in both methods. Model execution and gradient
acquisition are excluded throughout. Both dense and sparse methods evaluate
the exact local response.}
\label{fig:jointruntime}
\end{figure}

On the native case, scoring with prepared statistics takes 0.336\,ms for sparse
evaluation and 0.925\,ms for dense recomputation, a $2.75\times$ scoring speedup.
Including preparation, one strength takes 2.521\,ms versus 2.083\,ms, making
sparse evaluation 21.0\% slower. Three strengths sharing preparation take
3.292\,ms versus 4.014\,ms, an 18.0\% latency reduction
($1.22\times$ speedup). Across all seven cases, one-strength sparse evaluation
is 18.2--21.0\% slower, while three-strength reuse reduces latency by
18.0--21.8\% ($1.22$--$1.28\times$; Table~\ref{tab:jointruntimes}). Thus the
implementation benefits from reusing group statistics, with a measured
preparation cost. These results describe this eager implementation and hardware;
they do not establish a speed advantage over fused attention kernels.

\section{Positional prediction, decomposition, and precision repair}
\label{sec:positionalapplications}
\subsection{Protocol, tasks, and statistical units}
We evaluate the same frozen Qwen2.5-1.5B-Instruct checkpoint as
Section~\ref{sec:jointexperiment}, with its chat template,
FP32 eager attention, and grouped-query attention (12 query heads, 2 KV groups).
Calibration examines layers $\{0,7,14,21,27\}$, using 32 disjoint prompt sets from
each of four families: decisive-attribute inference, corrupted category mapping,
ordered register assignments, and SST-2 sentiment. The shared calibration policy
is fixed before the test evaluations reported here.

The held-out evaluation comprises 256 prompt sets per task, 768 in total:
\emph{Corrupt} asks for a two-category mapping from eight demonstrations with one
flipped label; \emph{Ordered} asks for the last assignment to a target register;
\emph{SST-2} uses eight training-set demonstrations, one label flipped, and a
held-out sentiment query from the validation split \cite{socher2013}.
SST-2 query examples and demonstration supports are sampled without reuse across
prompt sets, including calibration versus test. Each prompt set has four
demonstration orders and a fifth variant that appends 16 repetitions of an
irrelevant note before the query. Ordered-task labels follow the actual assignment
order. Calibration and test use separate deterministic streams with seed 24117.

Prediction errors and sign agreements are averaged within each prompt set and then
across the 256 sets for each task. Paired comparisons use the same prompt sets for
both methods. Reported 95\% intervals are pointwise percentile intervals. Positional
prediction summaries use 5,000 paired prompt-level bootstrap resamples;
frequency-interaction and repair summaries use 2,000. Multiple interventions and prompt
variants remain within their original prompt cluster. Answer NLL and KL are
computed on the model distribution conditional on the two labels A/B.

\subsection{Predicting the effects of finite positional interventions}
At layers $\{0,14,21\}$ retained by the shared calibration policy, we apply shifts
of 16 and 128 positional units to the cached keys of the demonstration tokens.
The 64 rotary planes are sorted by frequency and partitioned into bands of
22 fast, 21 middle, and 21 slow planes; an additional intervention shifts all
planes. Values, causal masks, and the current query are fixed at the edited
readout. Each edit is executed through the remaining model to measure its
downstream effect. The design supplies $5\times3\times2\times4=120$
interventions per prompt set and 30,720 per task. All predictors use the same
interventions and baseline downstream gradient.

Table~\ref{tab:phaseprediction} compares \eqref{eq:marginprediction} against
the measured answer-margin changes. The zero predictor has
$\widehat{\Delta m}=0$, so its MAE is the mean magnitude of the observed response.

\begin{table}[htbp]
\centering\small
\caption{Prediction of executed positional edits on 256 held-out prompt sets per task.
MAE is in downstream logit-margin units. Reductions compare finite and zero
MAE; brackets give paired 95\% prompt-bootstrap intervals for the percentage reduction.}
\label{tab:phaseprediction}
\setlength{\tabcolsep}{6pt}
\begin{adjustbox}{max width=\linewidth}
\begin{tabular}{lrrrr}
\toprule
Task & Finite MAE & Jacobian MAE & Zero MAE & Reduction vs.\ zero\\
\midrule
Corrupt & 0.02808 & 0.16028 & 0.05724 & 50.9\% [48.2, 53.7]\\
Ordered & 0.05786 & 0.21897 & 0.09073 & 36.2\% [35.0, 37.5]\\
SST-2 & 0.03585 & 0.15426 & 0.06681 & 46.3\% [43.1, 49.9]\\
\bottomrule
\end{tabular}
\end{adjustbox}
\end{table}

Finite prediction reduces MAE by 36.2--50.9\% relative to zero and
73.6--82.5\% relative to the positional Jacobian. Finite-predictor sign accuracy
is 96.52\% on Corrupt, 95.36\% on Ordered, and 96.40\% on SST-2; corresponding
Jacobian accuracies are 75.01\%, 74.32\%, and 73.12\%.
Paired finite-minus-zero MAE intervals are
$[-0.03040,-0.02789]$, $[-0.03401,-0.03174]$, and
$[-0.03291,-0.02906]$, respectively.

\subsection{Frequency, displacement, and layer resolve the response}
The angular displacement of plane $r$ is $\theta_r=\delta\omega_r$.
The saved model frequencies give the following scales:

\begin{table}[htbp]
\centering\small
\caption{Maximum absolute rotary angle in each frequency band, in radians.}
\label{tab:angles}
\begin{adjustbox}{max width=\linewidth}
\begin{tabular}{lrr}
\toprule
Band & Shift 16 & Shift 128\\
\midrule
Fast & 16.000 & 128.000\\
Middle & 0.1386 & 1.1084\\
Slow & 0.001489 & 0.011911\\
\bottomrule
\end{tabular}
\end{adjustbox}
\end{table}

At middle-band shift 16, the Jacobian reduces zero-predictor MAE by
94.0--94.8\%. The finite response further reduces its remaining error by
72.9--85.9\% (Table~\ref{tab:middleprediction}). At shift 128 the middle-band
response has mean magnitude 0.028--0.039; the finite predictor reduces Jacobian
error by 80.2--92.0\%. Each row averages 3,840 interventions from 256 prompts.

\begin{table}[htbp]
\centering\small
\caption{Middle-band prediction across two finite displacements. Absolute errors
and the zero baseline specify the response scale; the final column compares
finite and Jacobian MAE.}
\label{tab:middleprediction}
\setlength{\tabcolsep}{6pt}
\begin{adjustbox}{max width=\linewidth}
\begin{tabular}{lrrrrr}
\toprule
Task & Shift & Finite MAE & Jacobian MAE & Zero MAE & Reduction\\
\midrule
Corrupt & 16 & 0.000064 & 0.000236 & 0.004576 & 72.9\%\\
Ordered & 16 & 0.000057 & 0.000402 & 0.007160 & 85.9\%\\
SST-2 & 16 & 0.000040 & 0.000241 & 0.004042 & 83.4\%\\
\midrule
Corrupt & 128 & 0.003543 & 0.017862 & 0.031212 & 80.2\%\\
Ordered & 128 & 0.002239 & 0.027925 & 0.038921 & 92.0\%\\
SST-2 & 128 & 0.002112 & 0.016428 & 0.027545 & 87.1\%\\
\bottomrule
\end{tabular}
\end{adjustbox}
\end{table}

The paired 95\% intervals for the shift-16 reductions over the Jacobian are
$[71.1,74.6]\%$, $[85.4,86.3]\%$, and $[82.0,84.7]\%$ on Corrupt, Ordered,
and SST-2. In the slow band at shift 16, both predictors have MAE between
$4\times10^{-6}$ and $6\times10^{-6}$ at native FP32 execution precision, exhibiting their
agreement in the small-angle regime. The full band-by-shift means appear in
Appendix~\ref{app:bandstrata}.

Layer stratification further locates accurate downstream propagation. At layer
14, finite prediction reduces zero-predictor error by 80.1\%, 90.9\%, and
86.0\% on Corrupt, Ordered, and SST-2; at layer 21 the reductions are
94.7\%, 93.0\%, and 92.0\%. These summaries use every recorded band and shift
at the stated layer.

To describe responses with larger measured magnitude, we also stratify the
recorded interventions by $|\Delta m|\ge0.1$. This yields 4,285 Corrupt,
6,201 Ordered, and 4,195 SST-2 interventions, with all 256 prompts represented
in each task. Finite-predictor MAE is 42.2\%, 26.2\%, and 30.2\% below zero,
and its sign accuracy is 90.9\%, 85.6\%, and 89.0\%, respectively. Each prompt
is weighted equally after averaging its interventions in this response stratum.

\subsection{Measuring softmax coupling between rotary bands}
We evaluate the four local readouts in \eqref{eq:bandinteraction} on the same
pretrained query, cache, demonstration-token mask, and output projection.
The two bands and their union use exact rotations of the original keys.
For prompt $u$ and variant $v$, define the projected interaction and joint effect
\begin{equation}
 H_{uv}^{bc}=\sum_h\mathcal H_{bc,h}W_{o,h},\qquad
 T_{uv}^{bc}=\sum_h\bigl[F_h(s+d^b+d^c)-F_h(s)\bigr]W_{o,h},
 \label{eq:projectedinteraction}
\end{equation}
and report the mean relative interaction norm
\begin{equation}
 \overline\kappa_{bc}=\frac1{256}\sum_u\frac15\sum_{v=1}^{5}
 \frac{\norm{H_{uv}^{bc}}_2}{\norm{T_{uv}^{bc}}_2}.
 \label{eq:interactionmetric}
\end{equation}
Heads are summed after projection and before either norm is taken.

\begin{table}[htbp]
\centering\small
\caption{Fast--middle softmax interaction as a percentage of the combined
projected local response, using \eqref{eq:interactionmetric}. Each cell contains
256 prompt sets and five variants. Brackets give 95\% prompt-bootstrap intervals.}
\label{tab:bandinteraction}
\setlength{\tabcolsep}{9pt}
\begin{adjustbox}{max width=\linewidth}
\begin{tabular}{lrrr}
\toprule
Task & Layer & Shift 16 & Shift 128\\
\midrule
Corrupt & 0 & 9.20 [9.11, 9.28] & 23.92 [23.21, 24.58]\\
Corrupt & 14 & 3.33 [3.30, 3.36] & 18.96 [18.91, 19.01]\\
Corrupt & 21 & 4.78 [4.73, 4.82] & 17.22 [17.06, 17.41]\\
\midrule
Ordered & 0 & 5.31 [5.31, 5.32] & 26.94 [26.86, 27.02]\\
Ordered & 14 & 5.61 [5.59, 5.63] & 19.43 [19.41, 19.45]\\
Ordered & 21 & 4.96 [4.93, 4.99] & 16.20 [16.13, 16.27]\\
\midrule
SST-2 & 0 & 6.81 [6.61, 7.02] & 24.53 [23.82, 25.29]\\
SST-2 & 14 & 3.44 [3.38, 3.50] & 22.97 [22.67, 23.29]\\
SST-2 & 21 & 5.05 [4.90, 5.20] & 25.40 [24.91, 25.91]\\
\bottomrule
\end{tabular}
\end{adjustbox}
\end{table}

Fast--middle interaction grows from 3.33--9.20\% at shift 16 to
16.20--26.94\% at shift 128.
The exact score-additivity proposition identifies these measured interactions
with the mixed softmax response. The readouts quantify how the normalised
attention map couples independent positional score components.

\subsection{Separating positional and contextual changes under reordering}
For the three alternative demonstration orders without a gap, we align tokens
bijectively with the original prompt and construct four projected attention
readouts $o^{ab}$. Index $a$
selects the original or reordered positions; index $b$ selects the corresponding
observed de-rotated Q/K and V. Reapplying the selected rotations gives the four
cells, with $o^{00}$ and $o^{11}$ reproducing the two observed endpoints.
Define
\begin{equation}
\begin{aligned}
 \Phi&=o^{10}-o^{00},\qquad C=o^{01}-o^{00},\\
 J&=o^{11}-o^{10}-o^{01}+o^{00},\qquad
 T=o^{11}-o^{00}=\Phi+C+J.
\end{aligned}
\label{eq:phasecontent}
\end{equation}
Here $C$ includes changes in contextual Q/K/V produced by upstream computation.
For component $X\in\{\Phi,C,J\}$, the pooled signed projection is
$\gamma_X=\sum_r\langle X_r,T_r\rangle/\sum_r\norm{T_r}_2^2$,
where $r$ ranges over the aligned reordered readouts. These directional
contributions satisfy $\gamma_\Phi+\gamma_C+\gamma_J=1$.

\begin{table}[htbp]
\centering\small
\caption{Signed projections of direct phase, contextual-feature, and interaction
components onto the total projected change under actual reorderings. Each
task/layer contains 256 held-out prompt sets.}
\label{tab:phasecontent}
\setlength{\tabcolsep}{10pt}
\begin{adjustbox}{max width=\linewidth}
\begin{tabular}{lrrrr}
\toprule
Task & Layer & Phase $\gamma_\Phi$ & Context $\gamma_C$ & Interaction $\gamma_J$\\
\midrule
Corrupt & 0 & 1.000 & 0.000 & 0.000\\
Corrupt & 14 & 0.099 & 0.935 & -0.034\\
Corrupt & 21 & 0.028 & 0.964 & 0.009\\
\midrule
Ordered & 0 & 1.000 & 0.000 & 0.000\\
Ordered & 14 & -0.024 & 0.945 & 0.079\\
Ordered & 21 & -0.004 & 0.953 & 0.051\\
\midrule
SST-2 & 0 & 1.000 & 0.000 & 0.000\\
SST-2 & 14 & 0.035 & 0.967 & -0.002\\
SST-2 & 21 & 0.020 & 0.976 & 0.004\\
\bottomrule
\end{tabular}
\end{adjustbox}
\end{table}

At layer 0, direct phase accounts for essentially the entire projected
reordering response. At layers 14 and 21, the contextual-feature component
has signed projection 0.935--0.976. The exact decomposition therefore resolves
the transition from direct positional rotation to changes carried by the
representations built through the network.

\subsection{Attributing local cache deletions}
We independently suppress each of the eight demonstration spans at a fixed
attention readout, execute the resulting cache mask, and continue the model to
measure its downstream margin change. For a deleted set $G$, its headwise
attention mass $\alpha_G=\sum_{j\in G}p_j$ and weighted value
$w_G=\sum_{j\in G}p_jv_j$ give the exact finite response
\begin{equation}
 \Delta y_{-G}=\frac{\alpha_G y-w_G}{1-\alpha_G}.
 \label{eq:groupdeletion}
\end{equation}
An immediate cache-management consequence is the checkable constraint
$\norm{\sum_h\Delta y_{-G,h}W_{o,h}}_2\le\epsilon$: the chosen deletion
respects a specified current-query write-distortion budget. General finite
KV edits supply the corresponding criterion for quantization or
restoration. This is a local decision rule; future queries require their
own evaluation. The experiment below evaluates deletion-effect attribution.
Projecting and summing heads, then contracting with the baseline downstream
gradient, predicts the signed margin effect. One baseline backward pass supplies
the gradient for all eight candidate spans. The Jacobian control uses
$\alpha_Gy-w_G$ before the same projection and gradient contraction.
Attention mass and the norm of the projected weighted value supply magnitude
ranking controls.

At layers 0 and 14, we rank the eight candidates by predicted effect magnitude
and compute Spearman correlation with their measured absolute margin changes.
Correlations are averaged over the five variants within a prompt, then over
the 256 held-out prompts. Sign accuracy compares signed predictions with the
executed local interventions.

\begin{table}[htbp]
\centering\small
\caption{Local cache-deletion attribution against executed downstream effects.
The four central columns report magnitude-rank correlations; the final column
reports finite-predictor sign accuracy. Each row uses 256 prompt sets.}
\label{tab:deletionattribution}
\setlength{\tabcolsep}{5pt}
\begin{adjustbox}{max width=\linewidth}
\begin{tabular}{lrrrrrr}
\toprule
Task & Layer & Finite & Jacobian & Attention & Value norm & Sign accuracy\\
\midrule
Corrupt & 0 & 0.989 & 0.980 & 0.570 & 0.542 & 99.49\%\\
Corrupt & 14 & 0.950 & 0.933 & 0.723 & 0.727 & 98.43\%\\
Ordered & 0 & 0.988 & 0.974 & 0.377 & 0.361 & 99.13\%\\
Ordered & 14 & 0.982 & 0.981 & 0.757 & 0.764 & 99.43\%\\
SST-2 & 0 & 0.988 & 0.965 & 0.352 & 0.343 & 99.53\%\\
SST-2 & 14 & 0.977 & 0.974 & 0.701 & 0.714 & 99.15\%\\
\bottomrule
\end{tabular}
\end{adjustbox}
\end{table}

Finite-direction attribution achieves rank correlations of 0.950--0.989 and
98.4--99.5\% sign accuracy. The comparison connects the token-group decomposition
to separately executed model responses and provides a local effect-ranking
instrument alongside positional prediction.

\subsection{Targeted precision repair}
At layer 7, selected by lowest calibration mean relative repair RMSE, we apply
simulated symmetric per-vector 2-bit or
4-bit quantisation to the prefix K and V. For a vector $x$ and bit setting $b$,
the scale is $a=\max_k|x_k|/(2^{b-1}-1)$ and the stored test vector is
$a\,\operatorname{round}(x/a)$, with the zero vector handled separately.
Quantised values are represented in FP32 for the native intervention. The current
query's own key/value is retained at original precision. Repair is evaluated on
the original-order, no-gap variant of each of the 256 test prompt sets per task.

For a budget $k\in\{1,4\}$, exact greedy selection uses
\eqref{eq:repairbenefit} after every restoration. The first-order and second-order
controls expand \eqref{eq:singlerepair} along the joint score/value restoration
path, score candidates once in the initially corrupted state, and restore the
top $k$ entries. Attention mass, projected value norm, random selection, and a
Fisher-style score $\sum_h p_{hj}(1-p_{hj})\delta_{hj}^2$ provide additional
controls. All methods have access to the original cache. Thus the multi-entry
comparison evaluates exact greedy reselection against fixed-score selection;
the one-entry comparison evaluates a single choice for both methods.

For prompt $u$ and repaired combined write $\widehat o_u$, define
\begin{equation}
 r_u=\sqrt{\frac{\norm{\widehat o_u-o_u}_2^2}
                       {\norm{o^{\mathrm{corrupt}}_u-o_u}_2^2}},\qquad
 \overline r=\frac1{256}\sum_{u=1}^{256}r_u.
 \label{eq:repairmetric}
\end{equation}
Here $o_u$ is the original projected write; unrepaired corruption has $r_u=1$.
Heads are projected and summed before squared error is measured. The reported
quantity is the mean of prompt-level relative RMSEs.

\begin{table}[htbp]
\centering\small
\caption{Mean relative projected-write RMSE after restoring one or four prefix entries at layer 7. Each row averages 256 test prompt sets. Finite greedy recomputes benefits after each restoration; second-order scores are fixed at the initial corrupted state. Lower is better.}
\label{tab:precisionrepair}
\setlength{\tabcolsep}{6pt}
\begin{adjustbox}{max width=\linewidth}
\begin{tabular}{llrrrr}
\toprule
& & \multicolumn{2}{c}{One entry} & \multicolumn{2}{c}{Four entries}\\
\cmidrule(lr){3-4}\cmidrule(lr){5-6}
Task & Bits & Finite & Second-order & Finite & Second-order\\
\midrule
Corrupt & 2 & 0.7806 & 0.7838 & 0.5874 & 0.7794\\
Corrupt & 4 & 0.8807 & 0.9258 & 0.6173 & 0.8433\\
Ordered & 2 & 0.7697 & 0.8378 & 0.5932 & 0.7477\\
Ordered & 4 & 0.6811 & 0.6811 & 0.2870 & 0.6367\\
SST-2 & 2 & 0.7811 & 0.7815 & 0.5815 & 0.7835\\
SST-2 & 4 & 0.8073 & 0.8650 & 0.6079 & 0.7584\\
\bottomrule
\end{tabular}
\end{adjustbox}
\end{table}
\FloatBarrier

With one entry restored, both methods make a single choice from the same
corrupted cache and physical-token candidates. Exact selection reduces mean
relative RMSE by 4.86\% on Corrupt at 4 bits, 8.13\% on Ordered at 2 bits,
and 6.67\% on SST-2 at 4 bits. Their paired finite-minus-second-order intervals
are $[-0.05475,-0.03639]$, $[-0.06828,-0.06791]$, and
$[-0.07357,-0.04182]$, respectively. Across all six conditions, the finite
point estimate is lower in five and equal on Ordered at 4 bits.

With four entries restored, exact greedy selection lowers mean relative RMSE
by 19.8--54.9\% against fixed second-order selection across the six
task/precision conditions. All six paired difference intervals are below zero.
The one-entry comparison measures candidate choice, while the four-entry
comparison measures the complete repeated-restoration procedure. Because the
re-scoring schedules differ, the four-entry result does not isolate the benefit
of exact scoring from that of adaptive re-scoring.

\subsection{Numerical validation}
Native FP32 attention replay, physical KV edits, grouped-query layouts, batched
intervention isolation, and downstream gradients are checked before checkpoint
evaluation using tiny Qwen2 and Llama models. Float64 algebra checks and native
FP32 replay checks use separate references; the saved absolute/relative tolerances
are $2\times10^{-10}/2\times10^{-10}$ and
$5\times10^{-5}/5\times10^{-4}$, respectively. The three completed task workers
record 1,166,880 passing checks across their calibration and test stages, with
no failed saved checks. Their model revision, prompt manifest, numerical settings,
and shared calibration-policy hashes agree. Evaluation uses PyTorch
2.11.0+cu128 and Transformers 4.51.3; Corrupt and SST-2 run on A100 40GB and
Ordered on an L4. The complete checkpoints retain predictions, measured effects,
selection indices, per-prompt metrics, and named numerical checks.

\section{A recorded answer-repair case}
\label{app:answerrepair}
The accompanying instrument records one synthetic API-expression repair on
Qwen2.5-0.5B-Instruct, revision
\texttt{7ae557604adf67be50417f59c2c2f167def9a775}, using CPU FP32 eager
attention, PyTorch 2.8.0, and Transformers 4.57.3. Generation is greedy
with a 32-token limit and independent caches. Generated expressions are
mechanically qualified with the same API namespace in every branch before
running the same six fixed executable checks.

The ranking objective is the next-token logit margin for \texttt{legacy}
against \texttt{lookup}, supplied from the known API contract. The candidate
family contains seven odds multipliers, deletion, value zeroing, and six
band/shift combinations ($\pm16$ in fast, middle, and slow bands).
Scoring all token rows and layers gives 109,080 candidates from one captured
baseline and one backward pass, with no candidate-specific model calls for
scoring. The highest-ranked edit multiplies attention odds at prefix token
79, \texttt{\textvisiblespace legacy}, by 64 at layer 21 for the final query.
The predicted margin change is 11.361 and its executed change is 9.399.

\begin{table}[htbp]
\centering\small
\caption{Recorded outcomes on the same six fixed checks. Ordinary generation
and the attention branch use the same prompt. Source emphasis supplies the
full current API source. The installed no-op reproduces ordinary output
token for token. Expressions omit the mechanically added API namespace.}
\label{tab:answerrepair}
\begin{tabular}{llr}
\toprule
Branch & Generated expression & Checks passed\\
\midrule
Ordinary & \texttt{lookup(value)} & 0/6\\
Source emphasis & \texttt{current(value)} & 1/6\\
Installed no-op & \texttt{lookup(value)} & 0/6\\
Ranked attention edit & \texttt{legacy(current(value))} & 6/6\\
\bottomrule
\end{tabular}
\end{table}

The saved confirmation reports matching ordinary and attention prompt token
IDs, an unchanged check, and reproduction of the passing edited branch.
The case was selected adaptively and uses a supplied correct ranking target.
It demonstrates an executed repair under these conditions; it supplies no
success-rate estimate or claim that attention uniquely caused the failure.
The figure-data file retains the source record hash and reported branch
outcomes; full discovery, confirmation, and executable-case records belong
with the instrument's supporting records.

\section{Reference comparison with an independent regression learner}
\label{sec:regressionpilot}
This pilot compares trained prediction functions with the task-learning setup
of von Oswald et al.\ \cite{vonoswald2023}. It characterises the prediction
functions selected by training relative to an independently specified
regression-GD family. The exact gradient-step representation in
\eqref{eq:gradient} is established algebraically for arbitrary projection
weights and holds irrespective of alignment with this reference.

\paragraph{Controlled training protocol}
Following the task-learning comparison of von Oswald et al.\ \cite{vonoswald2023},
we train attention blocks on fresh linear-regression episodes. Each episode draws
$w_*\sim\mathcal N(0,I_5/5)$ and Gaussian inputs $x_j\sim\mathcal N(0,I_5)$,
with targets $t_j=x_j\T w_*$. Sixteen support tokens $[x_j,t_j]$ form the context;
an independent query is encoded as $[x,0]$. The query reads the support bank and
its prediction is the final coordinate of the residual output.

The three architectures are linear attention, softmax attention, and
RoPE--softmax attention. Each uses one residual attention block with two heads
of width 16 and learned query, key, value, and output projections. RoPE uses
base 10,000. We train seeds $0,1,2$ for 2,000 steps with AdamW, learning rate
$0.003$, zero weight decay, batches of 128 fresh episodes, and gradient-norm
clipping at 1. Architecture comparisons use matched training streams.
Checkpoint selection minimises prediction MSE on 128 validation episodes,
evaluated every 200 training steps. Held-out prediction evaluation uses
256 test episodes with eight independent queries per episode.

\paragraph{A separately specified one-step reference}
For support set $D=\{(x_j,t_j)\}_{j=1}^N$, define
\begin{equation}
 \mathcal L_D(w)=\frac1{2N}\sum_{j=1}^N(x_j\T w-t_j)^2,
 \qquad g_D=\frac1N\sum_{j=1}^N x_jt_j,
 \qquad w_1=0-\eta\nabla\mathcal L_D(0)=\eta g_D.
 \label{eq:independentgd}
\end{equation}
The resulting prediction is $f_{\eta,D}(x)=\eta x\T g_D$, and its query
Jacobian is $\nabla_x f_{\eta,D}=\eta g_D$. The regression rule is computed
from support inputs and labels independently of the attention projections,
scores, values, and effective matrices.

For each trained checkpoint, we calibrate one nonnegative scalar against its
validation predictions $f_\theta(D,x)$:
\begin{equation}
 \eta_{\mathrm{align}}=\operatorname*{arg\,min}_{\eta\ge0}
 \sum_{(D,x)\in\mathcal V}
 \bigl[f_\theta(D,x)-\eta x\T g_D\bigr]^2.
 \label{eq:alignedeta}
\end{equation}
This scalar is then fixed for test evaluation. Thus the reported comparison
measures held-out agreement with a validation-aligned GD family. A constructed
linear-attention positive control at $\eta=0.7$ matches the reference prediction
and query Jacobian to floating-point precision, with Jacobian cosine 1.0.

\paragraph{Held-out prediction and sensitivity alignment}
Table~\ref{tab:gdregression} reports every training seed. Prediction disagreement
is the squared difference between attention and the calibrated GD prediction,
averaged over the full test bank and then the three seeds. Jacobian cosines are
computed in float64 on sixteen held-out diagnostic episodes with eight queries
each, using the gradient of the native scalar prediction with respect to the
query's five input coordinates.

\begin{table}[htbp]
\centering\small
\caption{Courtesy regression comparison: trained attention versus the independently
specified one-step GD family, with a validation-calibrated scale. The cosine
columns show per-seed test means; prediction MSE is disagreement with the GD
reference, averaged across all three seeds on the full test bank.}
\label{tab:gdregression}
\setlength{\tabcolsep}{7pt}
\begin{adjustbox}{max width=\linewidth}
\begin{tabular}{lrrrr}
\toprule
& \multicolumn{3}{c}{Query-Jacobian cosine} & Prediction\\
\cmidrule(lr){2-4}
Architecture & Seed 0 & Seed 1 & Seed 2 & disagreement MSE\\
\midrule
Linear attention & 0.9951 & 0.9945 & 0.9974 & 0.005946\\
Softmax attention & 0.9628 & 0.9653 & 0.9632 & 0.026968\\
RoPE--softmax attention & 0.8231 & 0.9639 & 0.9647 & 0.072063\\
\bottomrule
\end{tabular}
\end{adjustbox}
\end{table}

The trained linear models attain cosine 0.994--0.997 with the regression-GD
direction. The ranges are 0.963--0.965 for softmax and 0.823--0.965 for
RoPE--softmax. The comparison describes functional alignment with an external
learner and its seed variability; it does not identify the objective selected
by training. The exact head representation is given by \eqref{eq:gradient}.

\section{Initial single-layer joint-KV experiment}
\label{app:jointinitial}
This earlier experiment uses 114-token retrieval prompts at one model and layer.
The supplemental sweep in Section~\ref{sec:jointexperiment} uses its own saved
112-token retrieval prompts and adds SST-2, another checkpoint size, and another
1.5B layer. The two runs are reported separately. Both use the predictors in
\eqref{eq:jointpredictor}--\eqref{eq:quadraticinteraction} and the prompt-paired
evaluation defined in Section~\ref{sec:jointexperiment}.

\subsection{Paired prompts and executed cache interventions}
We evaluate the interaction in \eqref{eq:jointmatrix} on frozen
Qwen2.5-1.5B-Instruct \cite{qwen15b2024}\footnote{Checkpoint revision:
\texttt{989aa7980e4cf806f80c7fef2b1adb7bc71aa306}.}
at zero-indexed layer 21, with FP32 eager attention, 12 query heads, and two
KV groups. The layer was fixed before this run using the positional-propagation
measurements reported in Section~\ref{sec:positionalapplications}.

The dataset contains 32 paired entity--label retrieval prompts generated with
seed 20260908. Each baseline lists eight records of the form
\texttt{Entity 06: B}, followed by a question asking for one entity's A/B label.
Its donor cyclically permutes entity identifiers and flips the label at every
record position, retaining the question and format. All prompts contain 114
tokens including the chat template. Tokenisation verifies identical record-span
positions at both endpoints. The eight record spans are the intervention
candidates; every pair and every span is evaluated.
Appendix~\ref{app:jointprotocol} gives a complete paired example.

For each span, we execute three patches: donor keys only, donor values only,
and donor keys and values together. Each replacement acts on all physical KV
groups at layer 21. It uses the native rotated keys and native values captured
from the donor prefix, while preserving the baseline query, self-token entry,
and every unpatched cache entry. The remaining query computation is executed
through the model. Throughout this experiment the target is the fixed margin
$m=\operatorname{logit}(A)-\operatorname{logit}(B)$, and the measured effect is
the patched margin minus the unpatched margin in the same counterfactual batch.

The complete design contains 256 joint patches and 512 key-only/value-only
controls. Capture uses 64 prefix forward passes and 32 baseline backward passes.
The 24 patches for each prompt share a counterfactual batch with one unpatched
control, giving 768 patched and 32 control single-query continuations.
Every batch row starts from an independent copy of the baseline cache.

\subsection{Initial prediction and interaction results}
\begin{table}[htbp]
\centering\small
\caption{Joint cache-intervention prediction on 32 paired prompts, eight spans
per prompt, at Qwen2.5-1.5B-Instruct layer 21. The response column is contracted
with the same baseline gradient after projecting and summing heads. MAE is in
answer-logit-margin units; recall and sign accuracy are percentages.
The zero predictor has undefined ranking and sign scores.}
\label{tab:jointattribution}
\setlength{\tabcolsep}{5pt}
\begin{adjustbox}{max width=\linewidth}
\begin{tabular}{llrrr}
\toprule
Predictor & Local response & MAE & Top-two recall & Sign accuracy\\
\midrule
Zero & $0$ & 0.079689 & --- & ---\\
Score/value first order & $L_s+V_0$ & 0.075484 & 64.06 & 81.25\\
Separate key/value & $K_f+V_0$ & 0.081529 & 65.63 & 83.98\\
Quadratic interaction & $K_f+V_0+C_2$ & 0.056043 & 68.75 & 85.16\\
\textbf{Exact joint} & $K_f+V_0+C_{KV}$ & \textbf{0.004075} & \textbf{96.88} & \textbf{99.22}\\
Dense equality control & $y'-y$ & 0.004075 & 96.88 & 99.22\\
\bottomrule
\end{tabular}
\end{adjustbox}
\end{table}

Including $C_{KV}$ reduces MAE by 95.0\% against separate key/value attribution
and by 92.7\% against the quadratic interaction correction
(Table~\ref{tab:jointattribution}). The primary difference is $-0.077455$,
with paired 95\% interval $[-0.093958,-0.062308]$. The joint-minus-quadratic
difference is $-0.051969$, with interval $[-0.065922,-0.039547]$.
Every prompt pair improves against both comparisons.
Figure~\ref{fig:jointresults}A displays the aggregate errors.
The exact joint predictor recovers 62 of the
64 top-two span slots and predicts the sign of 254 of the 256 effects.

\begin{figure}[htbp]
\centering
\includegraphics[width=\linewidth]{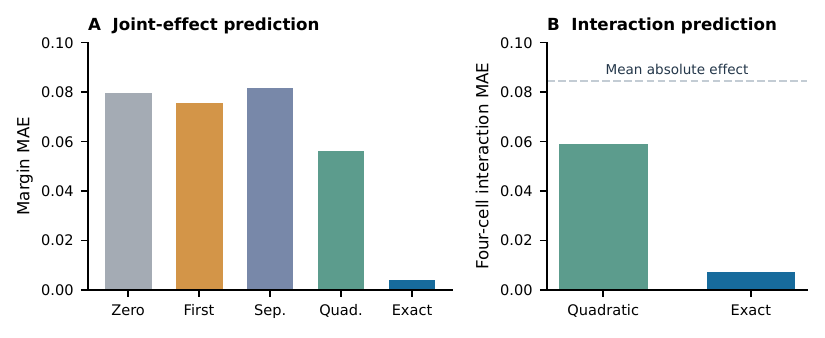}
\caption{Summary of the initial single-layer experiment. \textbf{A:} Reported
mean downstream prediction errors (Table~\ref{tab:jointattribution}).
\textbf{B:} Mean errors for predicting the executed four-cell interaction
with $\mathcal G(C_2)$ and $\mathcal G(C_{KV})$; the dashed line is the
mean absolute executed interaction. Bars are aggregate summaries over the
32 prompts and eight spans.}
\label{fig:jointresults}
\end{figure}

\subsection{The added term tracks an independently executed interaction}
The key-only and value-only controls isolate the two constituent responses.
For these controls $C_{KV}=0$, and the joint and separate predictors coincide:
their downstream MAEs are 0.005767 for key-only patches and 0.002422 for
value-only patches. Combining both edits introduces the interaction measured
in Table~\ref{tab:jointattribution}.

The executed margins supply the four-cell target in \eqref{eq:jointdownstreamfactorial}.
Across the 256 spans, $\widehat I_m$ has Pearson correlation 0.9829 with
$I_m$ and MAE 0.007424. Figure~\ref{fig:jointresults}B compares interaction
errors. The mean absolute
executed interaction is 0.084483; the quadratic prediction $\mathcal G(C_2)$
has interaction MAE 0.059019. These measurements connect the improvement
to the term identified algebraically: reallocating attention changes the effect
of a simultaneous value edit. The instrument provides separate signed key,
value, and interaction contributions for each candidate span.

\subsection{Numerical audit and reproducibility}
All 9,504 recorded intervention checks and 1,729 preflight checks pass. Independent
FP64 replay from the saved native tensors reproduces every joint prediction and
interaction score within $3\times10^{-15}$. Against the executed FP32 projected
write, the gradient-contracted local response differs by at most
$2.6\times10^{-6}$. The largest change in the unpatched margin between
single-query and counterfactual-batch execution is $4.6\times10^{-5}$; each
measured intervention uses its own batch's unpatched control.
Prompt texts, token spans, native Q/K/V, gradients, output projections, all
predictors, executed margins, and per-prompt summaries are retained for offline
replay. Appendix~\ref{app:jointprotocol} specifies the reproducible run.

\clearpage
\section{Supplemental joint-KV protocol and full results}
\label{app:jointsupplement}
\subsection{Checkpoints, prompts, and native execution}
The supplemental run uses runner version 1.0.1, seed 20260908,
PyTorch 2.11.0+cu128, Transformers 4.51.3, and CUDA 12.8 on an
NVIDIA A100-SXM4-40GB. Execution uses FP32 eager attention in evaluation mode,
frozen parameters, disabled TF32, and native RoPE at base $10^6$. The 0.5B
instruct checkpoint has 14 query heads, two KV heads, and head width 64;
the 1.5B instruct checkpoint has 12 query heads, two KV heads, and width 128.
The pinned checkpoint revisions are, respectively,
\texttt{7ae557604adf67be50417f59c2c2f167def9a775} and
\texttt{989aa7980e4cf806f80c7fef2b1adb7bc71aa306}.

The SST-2 snapshot is \texttt{stanfordnlp/sst2}, revision
\nolinkurl{8d51e7e4887a4caaa95b3fbebbf53c0490b58bbb}. Each query comes from the
validation split and each of its eight demonstrations from the training split,
balanced four per class. Query lengths are restricted to 20--320 characters and
demonstration lengths to 20--160 characters before tokenisation; normalised
duplicates within a prompt are excluded. The mapping is A for negative and B
for positive. Retrieval uses the entity permutation and label-flip rule in
Appendix~\ref{app:jointprotocol}, with its own saved 112-token chat prompts.
Both tokenizers validate all 64 prompt pairs before the sweep. The report
archive includes the exact texts, source row IDs, token spans, and their hashes.

The final query reads a captured native prefix cache. The runner clones that
cache for every counterfactual, patches native rotated keys and values, and
executes the remaining query computation. A leaf at the attention output
projection preserves the forward value while obtaining the baseline margin
gradient with respect to that projected write. Every batch includes an unpatched
control. Fractional edits save the actual rounded endpoint tensors used in both
prediction and execution. Completed setting--pairs are committed atomically
with hashes and the run fingerprint; the complete run records 192 such pairs.

The native checks compare full and cached continuation, unpatched controls,
head readouts, projected writes, and gradient contractions. FP64 algebra and
native FP32 comparisons use separate tolerances and references. The local
algebra tests also reconstruct the affine-input gradient-step identity directly.
The report archive contains all 5,120 scalar intervention records, 192 numerical
audit rows, per-prompt summaries, the CPU/CUDA gate records, and the timing
measurements. Native per-pair tensor captures are stored separately for offline
replay; the report archive alone does not contain those tensors.

\subsection{All full-replacement predictor means}
\begin{table}[htbp]
\centering\small
\caption{Supplemental full-replacement downstream margin MAE. Each row evaluates
32 prompt pairs and eight joint interventions per pair. Model sizes denote
Qwen2.5-Instruct checkpoints; layers are zero-indexed. Dense equality-control
MAEs coincide with exact joint to the displayed precision.}
\label{tab:jointcoverage}
\setlength{\tabcolsep}{4pt}
\begin{adjustbox}{max width=\linewidth}
\begin{tabular}{lllrrrrr}
\toprule
Model & Task & Layer & Zero & First order & Separate & Quadratic & Exact joint\\
\midrule
0.5B & Retrieval & 18 & 0.001521 & 0.001028 & 0.001042 & $2.98\times10^{-4}$ & $4.37\times10^{-6}$\\
0.5B & SST-2 & 18 & 0.007884 & 0.003183 & 0.003099 & $5.99\times10^{-4}$ & $2.64\times10^{-5}$\\
1.5B & Retrieval & 14 & 0.008275 & 0.003796 & 0.003342 & 0.001525 & $8.22\times10^{-5}$\\
1.5B & Retrieval & 21 & 0.098962 & 0.112774 & 0.107864 & 0.087683 & 0.007844\\
1.5B & SST-2 & 14 & 0.012270 & 0.004656 & 0.003194 & 0.001418 & $3.42\times10^{-4}$\\
1.5B & SST-2 & 21 & 0.108527 & 0.089914 & 0.109616 & 0.080700 & 0.004510\\
\bottomrule
\end{tabular}
\end{adjustbox}
\end{table}

\begin{table}[htbp]
\centering\footnotesize
\caption{Prompt-paired exact-minus-comparator MAE differences at full replacement.
Negative values favor exact joint. Intervals are individual 95\% percentile
intervals from 5,000 resamples of the 32 prompt pairs (seed 981), retaining all
eight spans. All 12 displayed intervals and the 12 additional intervals against
zero and first order lie below zero.}
\label{tab:jointpaired}
\setlength{\tabcolsep}{3pt}
\begin{adjustbox}{max width=\linewidth}
\begin{tabular}{lrrrr}
\toprule
Setting (model / task / layer) & $\Delta_{\mathrm{sep}}$ & 95\% interval & $\Delta_{\mathrm{quad}}$ & 95\% interval\\
\midrule
0.5B / Retrieval / 18 & -0.001037 & $[-0.001148,-0.000935]$ & -0.000293 & $[-0.000352,-0.000241]$\\
0.5B / SST-2 / 18 & -0.003072 & $[-0.003433,-0.002720]$ & -0.000572 & $[-0.000657,-0.000487]$\\
1.5B / Retrieval / 14 & -0.003260 & $[-0.003859,-0.002716]$ & -0.001442 & $[-0.001800,-0.001126]$\\
1.5B / Retrieval / 21 & -0.100020 & $[-0.129003,-0.076014]$ & -0.079840 & $[-0.105230,-0.056382]$\\
1.5B / SST-2 / 14 & -0.002852 & $[-0.003714,-0.002079]$ & -0.001076 & $[-0.001371,-0.000791]$\\
1.5B / SST-2 / 21 & -0.105106 & $[-0.120056,-0.090357]$ & -0.076190 & $[-0.091271,-0.060908]$\\
\bottomrule
\end{tabular}
\end{adjustbox}
\end{table}

\begin{table}[htbp]
\centering\small
\caption{Prediction of the independently executed four-cell margin interaction
$I_m$. The exact and quadratic predictions are $\mathcal G(C_{KV})$ and
$\mathcal G(C_2)$. Correlations use signed effects and are descriptive; each row
contains 256 spans clustered within 32 prompt pairs. The final column supplies
the absolute scale of the executed interaction.}
\label{tab:jointfactorial}
\setlength{\tabcolsep}{5pt}
\begin{adjustbox}{max width=\linewidth}
\begin{tabular}{lrrrr}
\toprule
Setting (model / task / layer) & Exact MAE & Quadratic MAE & Exact $r$ & Mean $|I_m|$\\
\midrule
0.5B / Retrieval / 18 & $4.52\times10^{-6}$ & $2.96\times10^{-4}$ & 0.99999 & 0.001043\\
0.5B / SST-2 / 18 & $2.54\times10^{-5}$ & $5.95\times10^{-4}$ & 0.99998 & 0.003094\\
1.5B / Retrieval / 14 & $6.00\times10^{-5}$ & 0.001513 & 0.99979 & 0.003321\\
1.5B / Retrieval / 21 & 0.008826 & 0.092810 & 0.99698 & 0.115820\\
1.5B / SST-2 / 14 & $1.62\times10^{-4}$ & 0.001331 & 0.99837 & 0.003308\\
1.5B / SST-2 / 21 & 0.003902 & 0.076860 & 0.99857 & 0.110007\\
\bottomrule
\end{tabular}
\end{adjustbox}
\end{table}

\subsection{Complete local runtime comparison}
All candidates are batched, with no automatic chunk fallback or precision change.
Synthetic cases use native-style Q/K/V tensors and eight-token spans; the native
capture has eight seven-token spans. PyTorch benchmark timing uses warmup,
accelerator synchronisation, five randomised rounds, and a minimum run time of
0.12 seconds per measurement. Totals directly measure the preparation and
evaluation specified in Section~\ref{sec:jointruntime}; component medians need
not sum to total medians. All seven timing audits pass against dense and FP64
references.

\begin{table}[htbp]
\centering\small
\caption{Directly measured local evaluation time including preparation (ms).
$R=1$ evaluates one strength; $R=3$ shares preparation across three strengths.
Ratios are dense time divided by sparse time, so a ratio below one denotes a
sparse slowdown. Values are medians of five timing rounds on one A100;
component timings, interquartile ranges, and raw measurements are archived.}
\label{tab:jointruntimes}
\setlength{\tabcolsep}{4pt}
\begin{adjustbox}{max width=\linewidth}
\begin{tabular}{lrrrrrr}
\toprule
& \multicolumn{2}{c}{$R=1$ (ms)} & \multicolumn{2}{c}{$R=3$ (ms)} & \multicolumn{2}{c}{Dense / sparse}\\
\cmidrule(lr){2-3}\cmidrule(lr){4-5}\cmidrule(lr){6-7}
Case & Dense & Sparse & Dense & Sparse & $R=1$ & $R=3$\\
\midrule
Native: $N=112$, $C=8$ & 2.083 & 2.521 & 4.014 & 3.292 & 0.83 & 1.22\\
Synthetic: $N=128$, $C=8$ & 2.096 & 2.517 & 4.021 & 3.296 & 0.83 & 1.22\\
Synthetic: $N=128$, $C=256$ & 2.133 & 2.533 & 4.146 & 3.279 & 0.84 & 1.26\\
Synthetic: $N=1024$, $C=8$ & 2.087 & 2.516 & 4.024 & 3.279 & 0.83 & 1.23\\
Synthetic: $N=1024$, $C=256$ & 2.134 & 2.537 & 4.157 & 3.288 & 0.84 & 1.26\\
Synthetic: $N=8192$, $C=8$ & 2.068 & 2.498 & 3.979 & 3.256 & 0.83 & 1.22\\
Synthetic: $N=8192$, $C=256$ & 2.142 & 2.531 & 4.197 & 3.281 & 0.85 & 1.28\\
\bottomrule
\end{tabular}
\end{adjustbox}
\end{table}

\clearpage
\section{Complete frequency-band prediction means}
\label{app:bandstrata}
Table~\ref{tab:allbandstrata} reports the recorded band-by-shift strata for
all three tasks. Each row averages 3,840 executed interventions from 256
prompt sets, using layers 0, 14, and 21 and all five prompt variants.
The zero-predictor MAE gives the mean absolute measured margin change.

\begin{table}[htbp]
\centering\small
\caption{Downstream positional-prediction MAE by frequency band and displacement,
in logit-margin units. All rows use the same finite, Jacobian, and zero predictors.}
\label{tab:allbandstrata}
\setlength{\tabcolsep}{8pt}
\begin{adjustbox}{max width=\linewidth}
\begin{tabular}{llrrrr}
\toprule
Task & Band & Shift & Finite MAE & Jacobian MAE & Zero MAE\\
\midrule
Corrupt & Fast & 16 & 0.001026 & 0.066144 & 0.024164\\
Corrupt & Fast & 128 & 0.108527 & 0.567056 & 0.166353\\
Corrupt & Middle & 16 & 0.000064 & 0.000236 & 0.004576\\
Corrupt & Middle & 128 & 0.003543 & 0.017862 & 0.031212\\
Corrupt & Slow & 16 & 0.000006 & 0.000006 & 0.000119\\
Corrupt & Slow & 128 & 0.000006 & 0.000009 & 0.000953\\
Corrupt & All & 16 & 0.001391 & 0.066527 & 0.026921\\
Corrupt & All & 128 & 0.110114 & 0.564376 & 0.203602\\
\midrule
Ordered & Fast & 16 & 0.089449 & 0.174306 & 0.140993\\
Ordered & Fast & 128 & 0.131197 & 0.681604 & 0.177739\\
Ordered & Middle & 16 & 0.000057 & 0.000402 & 0.007160\\
Ordered & Middle & 128 & 0.002239 & 0.027925 & 0.038921\\
Ordered & Slow & 16 & 0.000004 & 0.000004 & 0.000222\\
Ordered & Slow & 128 & 0.000005 & 0.000014 & 0.001787\\
Ordered & All & 16 & 0.080270 & 0.164970 & 0.138236\\
Ordered & All & 128 & 0.159647 & 0.702497 & 0.220769\\
\midrule
SST-2 & Fast & 16 & 0.002061 & 0.056467 & 0.022879\\
SST-2 & Fast & 128 & 0.143155 & 0.557336 & 0.218531\\
SST-2 & Middle & 16 & 0.000040 & 0.000241 & 0.004042\\
SST-2 & Middle & 128 & 0.002112 & 0.016428 & 0.027545\\
SST-2 & Slow & 16 & 0.000006 & 0.000006 & 0.000169\\
SST-2 & Slow & 128 & 0.000006 & 0.000010 & 0.001353\\
SST-2 & All & 16 & 0.002239 & 0.056709 & 0.024750\\
SST-2 & All & 128 & 0.137185 & 0.546868 & 0.235175\\
\bottomrule
\end{tabular}
\end{adjustbox}
\end{table}

\clearpage
\section{Joint-KV protocol and reproduction}
\label{app:jointprotocol}
The initial joint experiment in Appendix~\ref{app:jointinitial} fixes one model,
one layer, one dataset family, and one primary comparison before execution. For pair index $u\in\{0,\ldots,31\}$,
NumPy's random generator is initialised with
\texttt{SeedSequence([20260908, u])}. It samples eight A/B labels, ensures both
labels occur, shuffles them, and selects the query entity. The donor rotates
the entity-identifier array by a sampled nonzero cyclic shift and flips each
row's label. The complete prompts and token spans are saved before evaluation.
Every record includes its terminating newline in the patched span; the query's
own cache entry remains in the attended bank.

\begin{table}[htbp]
\centering\small
\caption{The first paired prompt. Both endpoints use the instruction
\emph{Read the records. Answer the question with A or B only.} and finish with
\emph{Question: What is the label of Entity 06?} followed by \emph{Answer:}.
The native Qwen chat template wraps the instruction, records, and question.}
\label{tab:jointexample}
\begin{adjustbox}{max width=\linewidth}
\begin{tabular}{rll}
\toprule
Record position & Baseline & Donor\\
\midrule
1 & \texttt{Entity 01: A} & \texttt{Entity 04: B}\\
2 & \texttt{Entity 02: B} & \texttt{Entity 05: A}\\
3 & \texttt{Entity 03: A} & \texttt{Entity 06: B}\\
4 & \texttt{Entity 04: B} & \texttt{Entity 07: A}\\
5 & \texttt{Entity 05: B} & \texttt{Entity 08: A}\\
6 & \texttt{Entity 06: B} & \texttt{Entity 01: A}\\
7 & \texttt{Entity 07: B} & \texttt{Entity 02: A}\\
8 & \texttt{Entity 08: A} & \texttt{Entity 03: B}\\
\bottomrule
\end{tabular}
\end{adjustbox}
\end{table}

For each record position the donor supplies the native keys and values at the
same token indices. All predictors receive identical captures. The fixed-value
key correction evaluates the complete span jointly, retaining its common
softmax denominator. The value correction uses baseline attention weights;
the joint correction uses the edited attention weights. For multiple heads,
the output projection and head sum precede the downstream-gradient contraction.
The complete per-intervention table reports both the signed $C_{KV}$ correction
and the executed factorial effect, together with the residual
$I_m-\mathcal G(C_{KV})$.

The run uses Transformers 4.51.3, PyTorch 2.11.0+cu128, and an NVIDIA A100
SXM4 40GB. Native execution is FP32 with TF32 disabled; local algebra and
effective-matrix audits use FP64. The effective readout audit uses the
mean-key score-shift convention, $k_j-\overline k$, and includes each endpoint's
uniform value mean. The baseline statistics and positive complement mass
implement the sparse formula in \eqref{eq:sparsejoint}.

The standalone command is
\begin{verbatim}
python rope_joint_kv_experiment.py --device cuda \
    --batch 32 --out joint_kv_results
\end{verbatim}
The supplied \texttt{Joint\_KV\_Attribution\_Colab.ipynb} embeds the same source.
Native tiny-Qwen preflights check cache isolation, grouped-query attention,
batch parity, full-forward parity, and the baseline gradient before the
pretrained checkpoint is downloaded. Each completed prompt pair is committed
with file hashes and the protocol fingerprint, enabling interrupted runs to
resume from the next pair.

The primary interval resamples prompt pairs with replacement 5,000 times
using bootstrap seed 981. The quadratic contrast uses the same sampling
scheme. The release tables retain all 768 interventions, the 32 prompt-level
summaries, and the model/source manifest. The figure-generation script reads
these tables directly. Native tensor captures additionally retain Q/K/V,
the output projection, the common gradient, batched controls, and the executed
outputs so each predictor can be reconstructed without rerunning the model.
\end{document}